\documentclass[english]{article}
\usepackage[latin9]{inputenc}
\usepackage{geometry}
\usepackage{babel}
\usepackage{array}
\usepackage{amsmath}
\usepackage{amsthm}
\usepackage{amssymb}
\usepackage{graphicx}
\usepackage[unicode=true,pdfusetitle,
 bookmarks=true,bookmarksnumbered=false,bookmarksopen=false,
 breaklinks=false,pdfborder={0 0 1},backref=false,colorlinks=false]
 {hyperref}

\makeatletter

\providecommand{\tabularnewline}{\\}

\theoremstyle{plain}
\newtheorem{thm}{\protect\theoremname}
\theoremstyle{definition}
\newtheorem{defn}[thm]{\protect\definitionname}

\usepackage{babel}

\providecommand{\definitionname}{Definition}
\providecommand{\theoremname}{Theorem}

\makeatother

\providecommand{\definitionname}{Definition}
\providecommand{\theoremname}{Theorem}

\begin{document}
\title{Numerical Methods For PDEs Over Manifolds Using Spectral Physics Informed
Neural Networks}
\author{Yuval Zelig \and Shai Dekel}
\date{School of Mathematical Sciences, Tel Aviv University \\
 \today}

\maketitle
 
\begin{abstract}
We introduce an approach for solving PDEs over manifolds using physics
informed neural networks whose architecture aligns with spectral methods.
The networks are trained to take in as input samples of an initial
condition, any time stamp and any point(s) on the manifold and then
output the solution's value at the given time and point(s). We provide
proofs of our method for the heat equation on the interval and examples
of unique network architectures that are adapted to nonlinear equations
on the sphere and the torus. We also show that our spectral-inspired
neural network architectures outperform the standard physics informed
architectures. Our extensive experimental results include generalization
studies where the testing dataset of initial conditions is randomly
sampled from a significantly larger space than the training set. 
\end{abstract}

\section{Introduction}

Time dependent differential equations are a basic tool for understanding
many processes in physics, chemistry, biology, economy and other fields.
Therefore, solving those equations is an active area of research \cite{key-1,key-2}.
For many of those equations, an analytical solution does not exist
and a numerical method must be used. Numerical methods such as finite
differences and finite elements methods are applied successfully in
many scenarios, however there remain many challenges. One still cannot
seamlessly incorporate noisy data into these algorithms, mesh generation
is complex, especially for the case of manifolds and solving high
dimensional problems governed by parameterized PDEs is sometimes out
of reach.

In recent years, there is an emergence of machine learning methods
and most notably Physics Informed (PI) deep learning models \cite{key-4},\cite{key-21}
that present an attractive alternative to the classical numerical
methods. PI machine learning allows to integrate seamlessly data and
mathematical physics models, even in partially understood, uncertain
and high dimensional contexts. Making a learning algorithm physics
informed amounts to introducing appropriate observational, inductive
or learning biases that can steer the learning process towards identifying
physically consistent solutions. In this work we are focused on Physics
Informed Neural Networks (PINN) that are designed to solve Partial
Differential Equations (PDEs) by enforcing the networks to approximately
obey the given governing equations. This can be achieved by applying
loss functions corresponding to the equations during the networks'
training phase. This technique allows to obtain relatively high quality
approximation without the need of ground truth data. There are various
neural network architectures that have been developed for this purpose,
with different settings and strategies such as automation differentiation
\cite{key-19}, numerical schemes \cite{key-5}, grid-free \cite{key-3,key-4}
or grid-dependent approaches \cite{key-5}, and the ability to handle
different geometries \cite{key-20}. In this paper, we present a generalization
of spectral based deep learning methods for PDEs \cite{key-24,key-25,key-26,key-27}: 
\begin{itemize}
\item[(i)] The architecture of our PINNs is guided by the paradigm of spectral
approximation over compact Riemannian manifolds, where on each manifold
we use the corresponding eigenfunction basis of the Laplace-Beltrami
operator. Introducing concepts from the theory of harmonic analysis
on manifolds to deep learning is an ongoing active research domain
\cite{key-16,key-22,MQS}. As we shall see, in our PDE applications,
this allows to construct neural networks that provide higher accuracy
using less parameters when compared with standard PINN architectures. 
\item[(ii)] Typically, PINNs need to be re-trained for each given initial condition,
whereas our approach can be considered an adaptation of PINNs to operator
learning \cite{LuLu,MQS}. It is en par with operator learning, as
it allows the network to take in as input any initial condition from
a fixed subspace of initial conditions over the manifold and output
the approximation to the PDE at a given point $x$ and time $t$.
However, the main advantage of the PI approach is that PINNs use the
PDE to construct the loss function for the training phase and so do
not require ground truth data for the training of the model. For operator
learning, one typically solves the target operator using classical
numerical methods in an offline stage and then trains the neural networks
using the obtained numeric solution \cite{LuLu}. Once trained, the
neural networks in both methods, provide fast, almost real-time, inference
for any given initial condition. 
\end{itemize}
The outline for the remainder of this paper is as follows. Section
2 reviews some preliminaries about PINNs and spectral approximation
over manifolds. Section 3 describes the key aspects of our approach.
In Section 4 we provide, as a pedagogical example, the theory and
details of the method for the simple case of the heat equation over
the unit interval. In Sections 5 and 6 we show how our approach is
applied for nonlinear equations over the sphere and torus. Our extensive
experimental results include generalization studies where the testing
dataset is sampled from a significantly larger space than the training
set. We also verify the stability of our models by injecting random
noise to the input and validating the errors increase in controlled
manner. Concluding remarks are found in Section 7.

\section{Preliminaries}

\subsection{Physics informed neural networks}

\label{subsec:PINN}

In this section, we describe the basic approach to PINNs presented
in \cite{key-4}. Generally, the goal is to approximate the solution
for a differential equation over a domain $\Omega$ of the form: 
\[
u_{t}+\mathcal{N}[u]=0,\quad t\in[0,T],
\]
with some pre-defined initial and/or boundary conditions. Typically,
a PINN $\tilde{u}(x,t)$ is realized using a Multi Layer Perception
(MLP) architecture. This is a pass forward network where each $j$-th
layer takes as input the vector $v_{j-1}$ which is the output of
the previous layer, applies to it an affine transformation $y=M_{j}v+b_{j}$
and then a coordinate-wise nonlinearity $\sigma$ to produce the layer's
output $v_{j}$ 
\begin{equation}
v_{j}=\sigma\circ(M_{j}v_{j-1}+b_{j}).\label{MLP-layer}
\end{equation}
In some architectures either the bias vector $b_{j}$ and/or the coordinate-wise
nonlinearity $\sigma$ are not applied in certain layers. In a standard
PINN architecture, the input to the network $\tilde{u}$ is $v_{0}=(x,t)$.
The unknown parameters of the network are the collection of weights
$\{M_{j},b_{j}\}_{j}$ and the network is trained to minimize the
following loss function: 
\[
MSE_{B}+MSE_{0}+MSE_{D},
\]
with the boundary value loss component 
\[
MSE_{B}=\frac{1}{N_{b}}\sum_{i=1}^{N_{b}}|\tilde{u}(x_{i}^{b},t_{i}^{b})-u(x_{i}^{b},t_{i}^{b})|^{2},
\]
the initial condition loss component 
\[
MSE_{0}=\frac{1}{N_{0}}\sum_{i=1}^{N_{0}}|\tilde{u}(x_{i}^{0},0)-u(x_{i}^{0},0)|^{2},
\]
and the differential loss component 
\[
MSE_{D}=\frac{1}{N_{d}}\sum_{i=1}^{N_{d}}|(\tilde{u}_{t}+\mathcal{N}[\tilde{u}])(x_{i}^{d},t_{i}^{d})|^{2}.
\]
In the above, $\{(x_{i}^{b},t_{i}^{b})\}_{i=1}^{N_{b}}$ is a discretized,
set of time and space points, where each $u(x_{i}^{b},t_{i}^{b})$
is the true given boundary value at $(x_{i}^{b},t_{i}^{b})$. The
set $\{x_{i}^{0}\}_{i=1}^{N_{0}}$, is a discretized set of possibly
randomized points in the domain and the initial condition $u(x,0)$
is given. The set $\{(x_{i}^{d},t_{i}^{d})\}_{i=1}^{N_{d}}$, typically
contains randomly distributed internal domain collocation points and
time steps. Since the architecture of the neural network is given
analytically (as in \eqref{MLP-layer} for the case of MLP), the value
$(\tilde{u}_{t}+\mathcal{N}[\tilde{u}])|_{(x_{i}^{d},t_{i}^{d})}$
at a data-point $(x_{i}^{d},t_{i}^{d})$ can be computed using the
automatic differentiation feature of software packages such as TensorFlow
and Pytorch \cite{key-6,key-7} (in our work we used TensorFlow).
Thus, the aggregated loss function enforces the approximating function
$\tilde{u}$ to satisfy required initial and boundary conditions as
well as the differential equation.

As we emphasized in the introduction, our approach is an adaptation
of PINNs to operator learning \cite{LuLu}, where the network is trained
to provide approximations to solutions for any initial conditions
from a given subspace. As we shall see, the adaptation requires applying
the PI loss functions for a given training set of initial conditions
(see the loss functions \eqref{L-0-loss},\eqref{L-D-loss}).

\subsection{Spectral decompositions over manifolds}

PDEs on manifolds appear in a variety of problems and applications
in fluid dynamics, material science, geophysics, solid mechanics,
control theory and biology. It is sometimes challenging to apply numerical
methods such as finite differences and finite elements, since grid
or mesh generation as well as discretizing the corresponding operators
is complex. Therefore, applying PINNs in these cases is potentially
attractive, since the method is essentially grid free and does not
require ground truth data for its learning process.

Following recent advancements in deep learning methods over manifolds
that use a spectral approach \cite{key-16,key-22}, in this work we
base our PINN design on approximation of the numeric PDE solutions
in the spectral domain. To this end we recall a fundamental result
in the spectral theory over manifolds regarding the spectrum of the
Laplace-Beltrami operator $\Delta$ and the spectral representation
of the solution to the heat equation 
\begin{thm}
\label{thm:general-spectral} \cite[Theorem 10.13]{key-8} Let $\Omega$
be a non-empty compact relatively open subset of a Riemannian manifold
$\mathcal{M}$ with metric $g$ and measure $\mu$. The spectrum of
$\mathcal{L}:=-\Delta$ on $\Omega$ is discrete and consists of an
increasing sequence $\{\lambda_{k}\}_{k=1}^{\infty}$ of non-negative
eigenvalues (with multiplicity) such that $\lim_{k\rightarrow\infty}\lambda_{k}=\infty$.
There is an orthonormal basis $\{\phi_{k}\}_{k=1}^{\infty}$ in $L_{2}(\Omega)$
such that each function $\phi_{k}$ is an eigenfunction of $-\Delta$
with eigenvalue $\lambda_{k}$. Moreover, the solution to the heat
equation $u_{t}=\Delta u$ on $\Omega$ with initial condition $u(x,t)=f(x),\ f\in L_{2}(\Omega)$,
is given by: 
\[
u(x,t)=\sum_{k=1}^{\infty}e^{-\lambda_{k}t}\langle f,\phi_{k}\rangle\phi_{k}(x).
\]
\\
\end{thm}

This well established result motivates the following spectral paradigm.
To solve the heat equation with some initial condition, one should
first decompose the initial condition function to a linear combination
of the eigenfunctions basis and then apply a time-dependent exponential
decay on the initial value coefficients. An approximation entails
working with the subspace spanned by $\{\phi_{k}\}_{k=1}^{K}$, for
some sufficiently large $K$ (see e.g. Theorem \ref{thm:4} below).
For a general manifold $\mathcal{M}$, the eigenfunctions do not necessarily
have an analytic form and need to be approximated numerically. As
we will show, we also follow the spectral paradigm for more challenging
cases of nonlinear equations over manifolds, where the time dependent
processing of the initial value coefficients is not obvious. Nevertheless,
a carefully crafted `spectral-inspired' architecture can provide superior
results over standard network architectures.

\section{The architecture of spectral PINNs}

\label{sec:spectralPINN}

Let $\mathcal{M}\subset\mathbb{R}^{n}$ be a Riemannian manifold,
$\Omega\subset\mathcal{M}$ a non-empty compact relatively open subset
and $\mathcal{N}$ a differential operator over this manifold, which
can possibly be nonlinear. We assume our family of initial conditions
comes from a subset $W\subset L_{2}(\Omega)$, of finite dimension,
that can be selected to be sufficiently large. Given a vector of samples
$\vec{f}$ of $f\in W$ over a fixed discrete subset of $\Omega$,
a point $x\in\mathcal{M}$ and $t\in[0,T]$, we would like to find
an approximation $\tilde{u}(\vec{f},x,t)$, given by a trained neural
network $\tilde{u}$, to the solution 
\[
u_{t}+\mathcal{N}[u]=0,
\]
\[
u(x,t=0)=f(x),\ \forall x\in\Omega.
\]

Recall that typically PI networks are trained to approximate a solution
for a single specific initial condition (such as in \cite{key-4}).
However, we emphasize that our neural network model is trained only
once for the family of initial conditions from the subspace $W$ and
that once trained, it can be used to solve the equation with any initial
condition from $W$. Moreover, as we demonstrate in our experimental
results, the trained network has the `generalization' property, since
it is able to approximate well the solutions when the initial value
functions are randomly sampled from a larger space containing $W$.

Our method takes inspiration from spectral methods for solving PDEs.
It is composed of 3 steps implemented by 3 blocks, as depicted in
Figure \ref{fig:General-description-of}: 
\begin{enumerate}
\item \textbf{Transformation Block - }The role of this block is to compute
from the samples $\vec{f}$ at specified locations of the initial
value condition $f\in W$ a `projection' onto $U_{K}=span\{\phi_{k}\}_{k=1}^{K}$,
for some given $K$, where $\{\phi\}_{k=1}^{\infty}$ are the eigenfunctions
of the Laplace-Beltrami operator on the manifold. We denote this block
as $\tilde{\mathcal{C}}:\vec{W}\rightarrow\mathbb{R}^{K}$, where
$\vec{W}$ is a subset of $\mathbb{R}^{L}$ which contains sampling
vectors of functions from $W$ over a fixed discrete subset of $\Omega$.
The desired output of the block is an estimation $\{\tilde{f}_{k}\}_{k=1}^{K}$
of the coefficients $\{\langle f,\phi_{k}\rangle\}_{k=1}^{K}$. However,
in cases where it is difficult to work with the spectral basis, one
can train an encoder to transform the input samples to a compressed
representation space of dimension $K$. Also, although the network
is trained on point samples of functions from $W$, it is able to
receive as input a sample vector $\vec{f}$ of a function $f$ which
is from a larger subset containing $W$ and approximate the solution.

Since generating a uniform or even quasi-uniform set of locations
on a manifold can be challenging, we emphasize that the advantage
of our learning approach is that the samples $\vec{f}$ can be taken
even from a set of random locations on $\Omega$, as long as the set
is consistently used for all initial conditions and is sufficiently
dense for the required accuracy. Indeed, our architecture preserves
one of the main advantages of PINNs, that they are grid-free. That
is, once trained, the networks can accept as input any parametric
point $x\in\Omega$ and any time $t\in[0,T]$, so as to provide the
grid-free approximation $\tilde{u}(\vec{f},x,t)$.

In most cases, it is advantageous to have the choice of the sampling
set and the quantities $L$ and $K$ to be determined by `Nyquist-Shannon'-type
theorems on the manifold for the given subset $W$ and the subspace
$U_{K}=\textrm{span}\{\phi_{k}\}_{k=1}^{K}$. In the scenario where
$W\subset U_{K}$ and the sampling set of size $L$ is selected to
provide perfect `Shanon'-type reconstruction, the transformation block
may take the form of a simple linear transformation. In complex cases,
where we have no prior knowledge about the required sampling rate
or we do not have perfect reconstruction from the samples, we train
a transformation block $\tilde{\mathcal{C}}$ that is optimized to
perform a nonlinear `projection' based on a carefully selected training
set. 
\item \textbf{Time Stepping Block -} In this block we apply a neural network
that takes as input the output of the transformation block $\tilde{\mathcal{C}}(\vec{f})$,
which may be the approximation of the spectral basis coefficients
$\{\tilde{f}_{k}\}_{k=1}^{K}$, and a time stamp $t$, to compute
a time dependent representation. We denote this block as $\tilde{\mathcal{D}}:\mathbb{R}^{K}\times[0,T]\rightarrow\mathbb{R}^{K}$. 
\item \textbf{Reconstruction Block -} In this block we apply an additional
neural network on the output of the time stepping block $\tilde{\mathcal{D}}$,
together with the given input point $x\in\Omega$, to provide an estimate
$\tilde{u}(\vec{f},x,t)$ of the solution $u(x,t)$ with the initial
condition $f$. We denote this block as $\mathcal{\tilde{R}}:\mathbb{R}^{K}\times\Omega\rightarrow\mathbb{R}$. 
\end{enumerate}
Thus, our method is in fact a composition of the 3 blocks $\tilde{u}:\vec{W}\times\Omega\times[0,T]\rightarrow\mathbb{R}$
\[
\tilde{u}(\vec{f},x,t)=\mathcal{\tilde{R}}(\tilde{\mathcal{D}}(t,\tilde{\mathcal{C}}(\vec{f})),x).
\]

Observe that in scenarios where one requires multiple evaluations
at different locations $\{\tilde{u}(\vec{f},x_{i},t)\}_{i}$, $x_{i}\in\Omega$,
at a given time step $t\in[0,T]$, one may compute once the output
of the time stepping block $\tilde{\mathcal{D}}(t,\tilde{\mathcal{C}}(\vec{f}))$
and use it multiple times for all $\{x_{i}\}_{i}$, and in doing so,
reduce the total computation time.

\begin{figure}
\begin{centering}
\includegraphics[width=12cm,height=8cm]{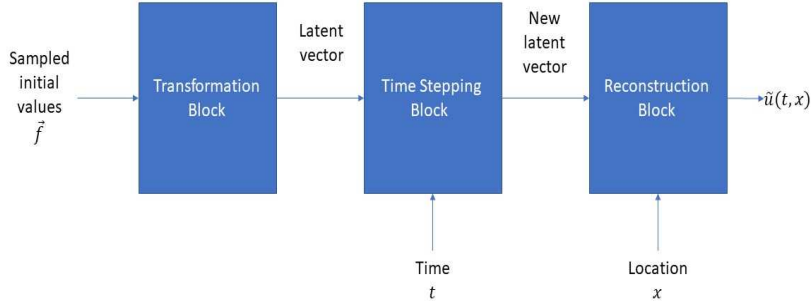} 
\par\end{centering}
\caption{\label{fig:General-description-of}General description of our method}
\end{figure}

\section{Introduction of the spectral PINN for the heat equation over $\Omega=[0,1]$}

We first review the prototype case of the heat equation on the unit
interval where we can provide rigorous proofs for our method as well
as showcase simple realization versions of our spectral network construction.
Recall the heat equation: 
\[
u_{t}=\alpha u_{xx},\quad x\in[0,1],t\in[0,0.5],
\]
with initial time condition: 
\[
u(x,t=0)=f(x),\ x\in[0,1].
\]

\subsection{Architecture and theory for the heat equation over $\Omega=[0,1]$}

The analytic solution to this equation can be computed in 3 steps
that are aligned with the 3 blocks of our architecture. Assume the
initial condition $f:[0,1]\rightarrow\mathbb{R}$ has the following
spectral representation 
\[
f(x)=\sum_{k=1}^{\infty}c_{k}\sin(2\pi kx).
\]
Next, apply the following transformation on the coefficients for a
given time step $t$ 
\[
\mathcal{D}(t,c_{1},c_{2},...):=(e^{-4\pi^{2}\alpha t}c_{1},e^{-4\pi^{2}\cdot2^{2}\alpha t}c_{2},...).
\]
Finally, evaluate the time dependent representation at the point $x$:
\[
u(x,t)=\mathcal{R}(e^{-4\pi^{2}\alpha t}c_{1},e^{-4\pi^{2}\cdot2^{2}\alpha t}c_{2},...,x):=\sum_{k=1}^{\infty}e^{-4\pi^{2}k^{2}\alpha t}c_{k}\sin(2\pi kx).
\]
We now proceed to provide the details of the numerical spectral PINN
approach in this scenario. First, we select as an example $K=20$
and $W=W_{20}$, where 
\[
W_{20}:=\left\{ {\sum_{k=1}^{20}c_{k}\sin(2\pi kx),\quad c_{1},...,c_{20}\in[-1,1],\sqrt{c_{1}^{2}+...+c_{20}^{2}}=1}\right\} .
\]
We sample each $f\in W_{20}$ using $L=101$ equally spaced points
in the segment $[0,1]$ to compute a vector $\vec{f}$. For the training
of the networks we use a loss function which is a sum of two loss
terms $L_{0}+L_{D}$. The loss $L_{0}$ enforces the network $\tilde{u}_{\theta}$
with weights $\theta$ to satisfy $N_{0}$ random training initial
conditions 
\begin{equation}
L_{0}(\theta)=\frac{1}{101N_{0}}\sum_{i=1}^{N_{0}}\sum_{j=0}^{100}\left|{\tilde{u}_{\theta}\left({\vec{f_{i}},\frac{j}{100},0}\right)-f_{i}\left({\frac{j}{100}}\right)}\right|^{2}.\label{L-0-loss}
\end{equation}

For the second loss term we randomly generate $N=5,000$ triples $(\vec{f}_{i},x_{i},t_{i})_{i=1}^{N}$
and enforce the model to obey the differential condition 
\begin{equation}
L_{D}(\theta)=\frac{1}{N}\sum_{i=1}^{N}\left|\frac{\partial\tilde{u}_{\theta}(\vec{f_{i}},x_{i},t_{i})}{\partial t}-\alpha\frac{\partial^{2}\tilde{u}_{\theta}(\vec{f_{i}},x_{i},t_{i})}{\partial x^{2}}\right|^{2}.\label{L-D-loss}
\end{equation}
The derivatives of the given neural network approximation in \eqref{L-D-loss}
are calculated using the automatic differentiation capabilities of
deep learning frameworks. In this work we use TensorFlow \cite{key-6}.

Observe that although we are using in this pedagogical example a uniform
grid for the samples of the initial conditions, as explained in Section
\ref{sec:spectralPINN}, the advantage of our learning approach is
that the samples $\vec{f}$ can be taken even from a set of random
locations, as long as the set is consistently used for all initial
conditions during training and inference and is sufficiently dense
for the required accuracy.

We compare two PINN architectures that provide an approximation to
the solution $u$: 
\begin{enumerate}
\item[(i)] \textbf{The naive model -} We benchmark our spectral method with
a deep learning model which is based on a standard MLP neural network,
that takes in as input $(\vec{f},t,x)\in\mathbb{R}^{103}$ and outputs
an approximation. This model is trained to be PI using the loss function
$L_{0}+L_{D}$, where the two terms are defined in \eqref{L-0-loss}
and \eqref{L-D-loss}. The network is composed of 5 dense layers $\mathbb{R}^{103}\rightarrow\mathbb{R}^{103}$
and finally a dense layer $\mathbb{R}^{103}\rightarrow\mathbb{R}$.
Each of the first five dense layers is followed by a non-linear activation
function. Typically, a Rectifier Linear Unit (ReLU) $\sigma(x)=(x)_{+}$,
is a popular choice as the nonlinear activation for MLP networks \cite{key-12}.
However, it is not suitable in this case, since its second derivative
is almost everywhere zero. Therefore we use $\tanh$ as the nonlinear
activation function. Observe that in this paper, the naive PINN model
differs from the classic PINN model reviewed in Subsection \ref{subsec:PINN},
in that it is trained to approximate the solution for any initial
condition from the given subspace. 
\item[(ii)] \textbf{The spectral model -} In some sense, our spectral model $\tilde{u}$
is `strongly' physics informed. Exactly as the naive model, it is
also trained using the loss functions \eqref{L-0-loss} and \eqref{L-D-loss},
to provide solutions to the heat equation. However, its architecture
is different from the naive architecture, in that it is modeled to
match the spectral method. The spectral model $\tilde{u}$ approximates
$u$ using the 3 blocks of the spectral paradigm approximation presented
in the previous section. We now provide the details of the architecture
and support our choice of design with rigorous proofs 
\end{enumerate}
\begin{enumerate}
\item[1.] \textbf{Sine transformation block} This block receives as input a
sampling vector $\vec{f}$ and returns the sine transformation coefficients
for $\{\sin(2\pi k\cdot)\}$, $k=1,\dots,20$. Due to the high sampling
rate $L=101$, compared with the frequency used $K=20$, the sampled
function $f$ can be fully reconstructed from $\vec{f}$ and this
operation can be realized perfectly using the Nyquist-Shannon sampling
formula. However, so as to simulate a scenario on a manifold where
the sampling formula cannot be applied, we train a network to apply
the transformation. To this end, we created $1,000$ initial value
conditions using trigonometric polynomials of degree 20, and trained
this block to extract the coefficients of those polynomials. In other
words, we pre-trained $\tilde{\mathcal{C}}:\mathbb{R}^{101}\rightarrow\text{\ensuremath{\mathbb{R}^{20}}}$
for the following task: 
\[
\tilde{\mathcal{C}}(\vec{f})=(c_{1},...,c_{20}),
\]
where $\vec{f}$ is the sampling vector of the function 
\[
f(x)=\sum_{k=1}^{20}c_{k}\sin(2\pi kx).
\]
In this simple case where $\Omega=[0,1]$, the network can simply
be composed of one dense layer with no nonlinear activation, which
essentially implies computing a transformation matrix from samples
to coefficients. As already noted in the introduction, for manifolds
such as the embedded torus (see Section \ref{sec:torus}), where the
spectral basis can only be computed numerically, or a nonlinear encoder
is trained to `simulate' the spectral basis, the architecture of the
transformation block is more complex. 
\item[2.] \textbf{Time stepping block} The time stepping block should approximate
the function: 
\begin{equation}
\mathcal{D}(t,c_{1},...,c_{20})=(e^{-4\pi^{2}\alpha t}c_{1},e^{-4\pi^{2}\cdot2^{2}\alpha t}c_{2},...,e^{-4\pi^{2}20^{2}\alpha t}c_{20}).\label{realize-D}
\end{equation}
We consider 2 architectures for this block:

\textbf{Realization time stepping block:}\\
 In the case of the heat equation we know exactly how the time stepping
block should operate and so we can design a true realization. The
first layer computes 
\[
t\rightarrow(-4\pi^{2}\alpha t,-4\pi^{2}2^{2}\alpha t,...,-4\pi^{2}20^{2}\alpha t).
\]
The second layer applies the exponential nonlinearity 
\[
(-4\pi^{2}\alpha t,-4\pi^{2}\cdot2^{2}\alpha t,...,-4\pi^{2}20^{2}\alpha t)\rightarrow(e^{-4\pi^{2}\alpha t},e^{-4\pi^{2}\cdot2^{2}\alpha t},...,e^{-4\pi^{2}20^{2}\alpha t}).
\]
Finally, we element-wise multiply the output of the second layer with
$(c_{1},...,c_{20})$ to output the time dependent spectral representation
\eqref{realize-D}.

\textbf{Approximate time stepping block:}\\
 In the case of general manifolds we may not be able to fully realize
the time stepping block. Therefore, we examine what are the consequences
of using an MLP network $\mathcal{\tilde{D}}$ that approximates for
given $K\ge1$ 
\[
\mathcal{D}(t,c_{1},...,c_{K}):=(e^{-4\pi^{2}\alpha t}c_{1},...,e^{-4\pi^{2}K^{2}\alpha t}c_{K}).
\]
The fact that the operator $\mathcal{D}$ is a composition of analytic
components allows us to construct relatively small approximating NN
as we prove in the following theorem (see the appedix for proofs): 
\begin{thm}
\label{thm:2}For any $0<\epsilon<1$ and $K\ge1$ there exists a
MLP network $\mathcal{\tilde{D}}$, consisting of dense layers and
$\tanh$ as an activation function, with $O(K^{3}+K\log^{2}(\epsilon^{-1}))$
weights such that 
\[
\|\mathcal{\tilde{D}}(t,c_{1},...,c_{K})-\mathcal{D}(t,c_{1},...,c_{K})\|_{\infty}\leq\epsilon,
\]
for all inputs $c_{1},...,c_{K}\in[-1,1],t\in[0,1]$. 
\end{thm}

We remark that it is possible to approximate $\mathcal{D}$ using
ReLU as the nonlinear activation as it shown in \cite{key-13}. However,
recall the ReLU is not suitable for our second order differential
loss function \eqref{L-D-loss}. In the experiments below, the approximating
MLP time stepping block is composed of 5 layers. 
\item[3.] \textbf{Reconstruction Block} The reconstruction block should operate
as follow: 
\[
\mathcal{R}(a_{1},....,a_{K},x)=\sum_{k=1}^{K}a_{k}\sin(2\pi kx).
\]
In the case of the heat equation, for given $t\in[0,1]$, the coefficients
$\{a_{k}\}_{k=1}^{K}$ are $\{e^{-4\pi^{2}k^{2}t}c_{k}\}_{k=1}^{K}$
or an approximation to these coefficients. Here, also one can design
a realization block which uses the sine function as a nonlinearity.
To support the general case we have the following result 
\begin{thm}
\label{thm:3} For fixed $A>0$, $K\ge1$ and any $0<\epsilon<1$,
there exists a MLP network $\tilde{\mathcal{R}}$, consisting of dense
layers and $\tanh$ as an activation function, with $O(K^{2}+K\log^{2}(K\epsilon^{-1}))$
weights for which 
\[
|\mathcal{\tilde{\mathcal{R}}}(a_{1},....,a_{K},x)-\mathcal{R}(a_{1},....,a_{K},x)|\leq\epsilon,
\]
where $a_{1},...,a_{K}\in[-A,A],x\in[0,1]$. 
\end{thm}

\end{enumerate}
In the experiments below, the approximating MLP reconstruction block
is composed of 5 layers. Using theorem \ref{thm:2} and \ref{thm:3},
we can prove a general theorem that provides an estimate for the approximation
of a MLP network. We first give the definition of Sobolev spaces \cite{key-23}:\\

\begin{defn}
Let $\Omega\subset\mathbb{R}^{n}$ and $C_{0}^{r}(\Omega)$ be the
space of continuously $r$-differentiable with compact support functions.
For $1\leq p<\infty$, the Sobolev space $W_{p}^{r}(\Omega)$ is the
completion of $C_{0}^{r}(\Omega)$ with respect to the norm 
\[
\|f\|_{W_{p}^{r}(\Omega)}=\sum_{|\alpha|\leq r}\|\partial^{\alpha}f\|_{L_{p}(\Omega)},
\]
where $\partial^{\alpha}f=\frac{\partial^{|\alpha|}f}{\partial x_{1}^{\alpha_{1}}...\partial x_{n}^{\alpha_{n}}},|\alpha|=\sum_{i=1}^{n}\alpha_{i}$. 
\end{defn}

With this definition at hand we are ready to state a result on the
approximation capabilities of our spectral architecture when MLP networks
are used to approximate the spectral realization 
\begin{thm}
\label{thm:4} Let $r\in\mathbb{N}$. For any $0<\epsilon<1$ there
exists a MLP neural network $\tilde{u}$, with $\tanh$ nonlinearities
and $O(\epsilon^{-3/r}+\epsilon^{-1/r}\log^{2}(\epsilon^{-(1+1/r)}))$
weights (the constant depends on $r$) for which the following holds:
For any $f\in W_{2}^{r}([0,1])$, $f=\sum_{k=1}^{\infty}c_{k}\sin(2\pi kx)$,
$\|f^{(r)}\|_{2}\leq1$ and $u$, the solution to the heat equation
on $\Omega=[0,1]$ with the initial condition $f$, the network $\tilde{u}$
takes the input $\{c_{k}\}_{k=1}^{K}$, $K\ge c\epsilon^{-1/r}$ and
provides the estimate 
\[
\|u(f,\cdot,t)-\tilde{u}(f,\cdot,t)\|_{L_{2}[0,1]}\leq\epsilon,\qquad\forall t\in[0,1].
\]
\end{thm}

\subsection{Experimental Results}

\label{subsec:heat-experiment}

In our experiments we tested 4 PINN models. The first is a naive PINN
model with vanilla MLP architecture consisting of 6 layers. This model
accepts as input the samples of the initial condition, a point $x\in[0,1]$
and time step $t$ and outputs an approximation to the solution. We
then also tested 3 variations of the spectral model with the various
blocks realized or approximated. Training was performed using $5,000$
and $25,000$ samples of the form $(\vec{f},x,t)$, where $\vec{f}$
is a sampling vector of trigonometric polynomial of degree 20 on 101
equispaced points in the segment $[0,1]$ with $t\in[0,0.5]$. To
guarantee slow vanishing of the solution over time we used $\alpha=0.01$.
The testing of the 4 models was done using 20 randomly sampled initial
conditions. For each model, we measured the Mean Squared Error (MSE)
of the approximated solutions for these initial conditions over 500
uniform time steps and 101 uniform locations. We also tested an operator
learning model of the type `Unstacked DeepONet' \cite{LuLu}. This
is a network that is trained without knowledge of the PDE and therefore
requires for its training phase ground truth data of the training
solutions. A `branch' subnet processes the input samples of the initial
condition, a `trunk' subnet processes the $x$ and $t$ parameters
and then the output of these two subnets is further processed to provide
the approximation. We found that operator learning requires a significantly
larger network and this in turn necessitates a larger training set
of 50,000 samples. The results are summarized in Table \ref{tab:segment res}.
We see that a network that realizes the spectral method performs best.
When approximations replace realization components then they still
outperform standard architectures.

\begin{table}
\begin{centering}
\begin{tabular}{|>{\centering}p{1.5cm}|>{\centering}p{5cm}|c|>{\centering}p{1.5cm}|>{\centering}p{1.5cm}|>{\centering}p{1.5cm}|}
\hline 
Model number in plots  & Model Architecture  & \#Model weights  & Testing MSE: 5,000 training samples  & Testing MSE: 25,000 training samples  & Testing MSE: 50,000 training samples \tabularnewline
\hline 
\hline 
1  & Naive Model  & 53,664  & 1.3e-4  & 1.19e-4  & N/A \tabularnewline
\hline 
2  & Spectral model - full realization (time stepping and reconstruction
blocks)  & \textbf{2,960}  & \textbf{9.0e-6} & \textbf{8.3e-6}  & N/A \tabularnewline
\hline 
3  & Spectral model - MLP approximation of time stepping block, realization
of reconstruction block  & 11,980  & 5.7e-5  & 4.9e-5  & N/A \tabularnewline
\hline 
4  & Spectral model - realization of time stepping block, MLP approximation
of reconstruction block  & 10,401  & 2.9e-5 & 2.87e-5  & N/A \tabularnewline
\hline 
5  & Operator learning `Unstacked DeepONet' \cite{LuLu}  & 998,102  & N/A  & N/A  & 9.81e-5 \tabularnewline
\hline 
\end{tabular}
\par\end{centering}
\caption{Heat equation over $\Omega=[0,1]$ - Comparison of a standard naive
PINN model, 3 variants of our spherical PINN model and operator learning.}
\label{tab:segment res} 
\end{table}

In Figure \ref{fig:Error-versus-time segment} we plot over different
time steps, the sums over 20 test cases of mean squared error between
the approximation of the network $\tilde{u}$ and the ground truth
$u$. 
\[
Error(t)=\sum_{i=1}^{20}\frac{1}{101}\sqrt{\sum_{k=0}^{100}\left|{\tilde{u}\left({\vec{f}_{i},\frac{k}{100},t}\right)-u\left({f_{i},\frac{k}{100},t}\right)}\right|^{2}}.
\]
We show some examples of the exact solution $u$ and the approximations
of the different variants of neural network at different times and
with several initial condition in figure \ref{fig:examples-of-exact segment}.
\begin{figure}
\centering{}\includegraphics[width=12cm,height=8cm]{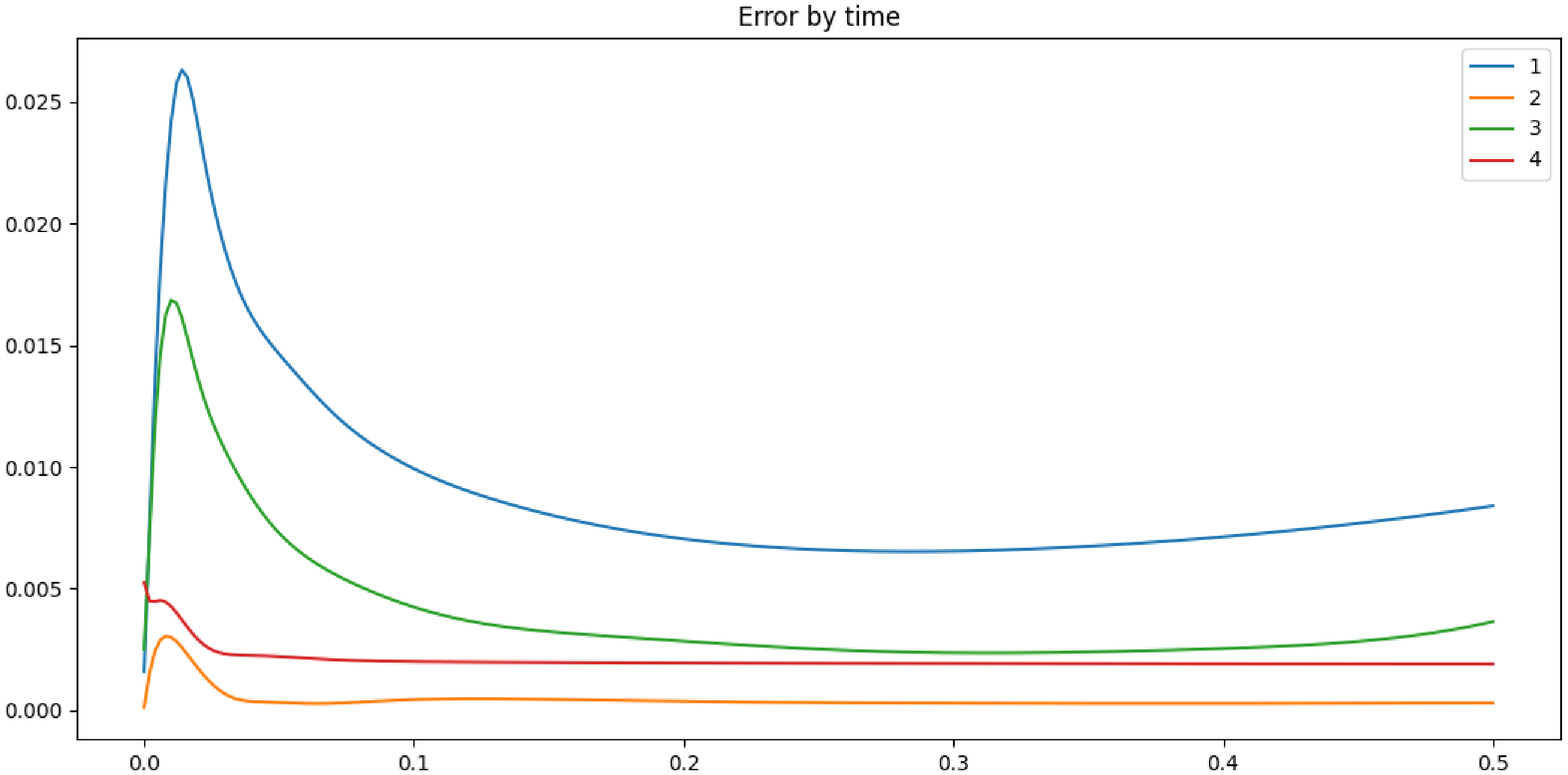}\caption{\label{fig:Error-versus-time segment}Heat equation on $[0,1]$ -
Error versus time}
\end{figure}

\begin{figure}
\begin{centering}
\includegraphics[width=12cm,height=8cm]{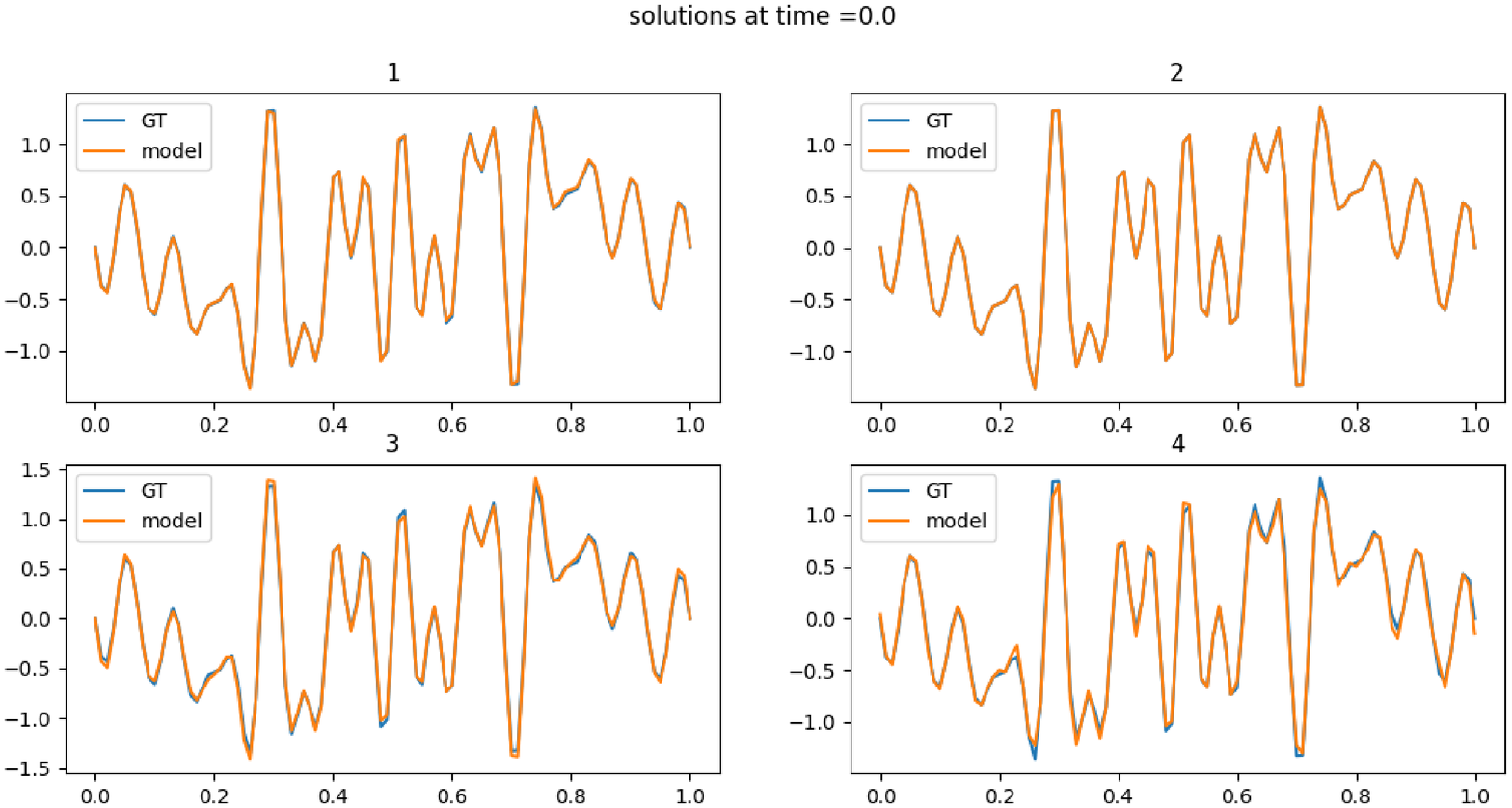}\\
 \includegraphics[width=12cm,height=8cm]{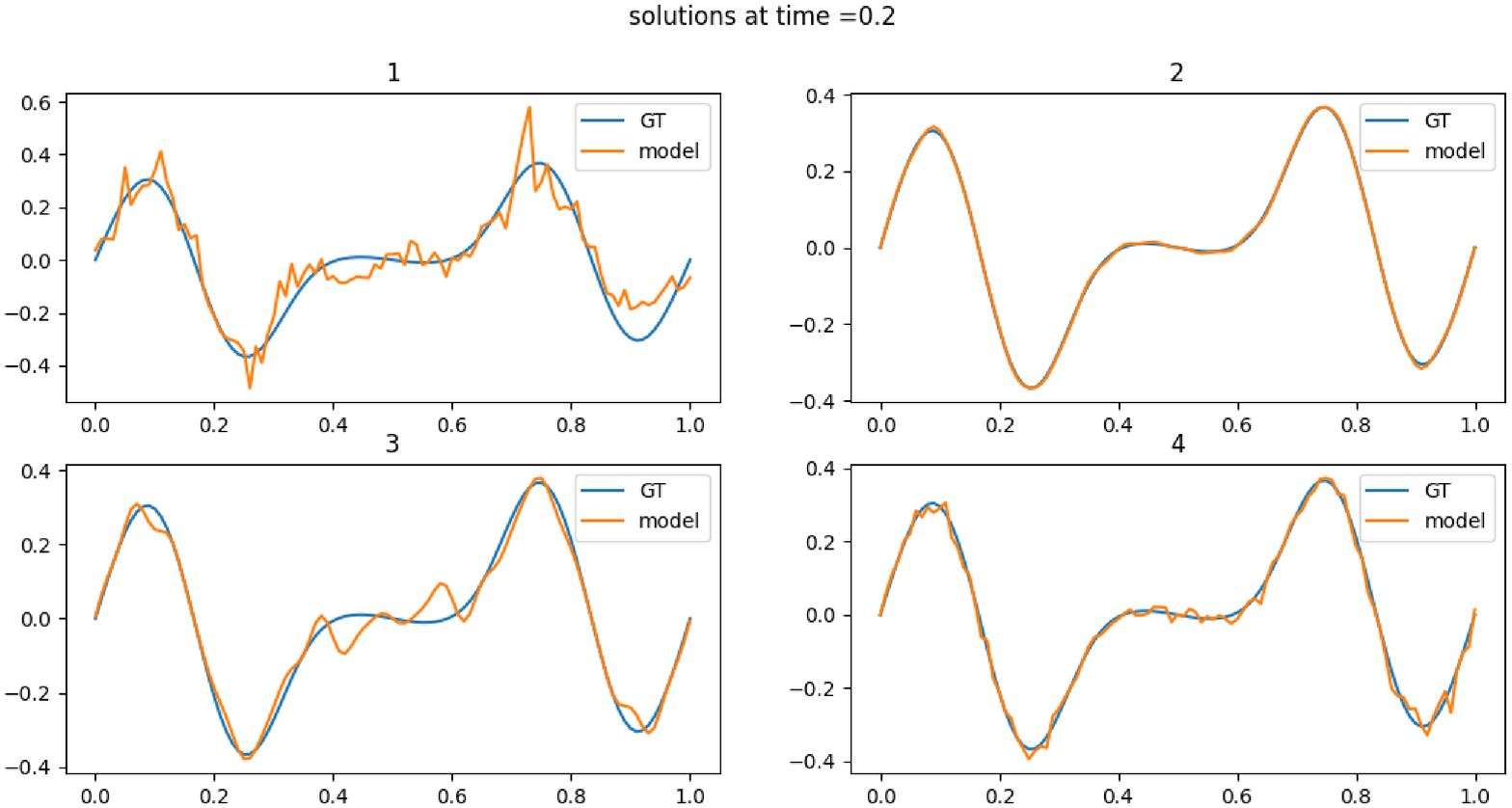}\\
 \includegraphics[width=12cm,height=8cm]{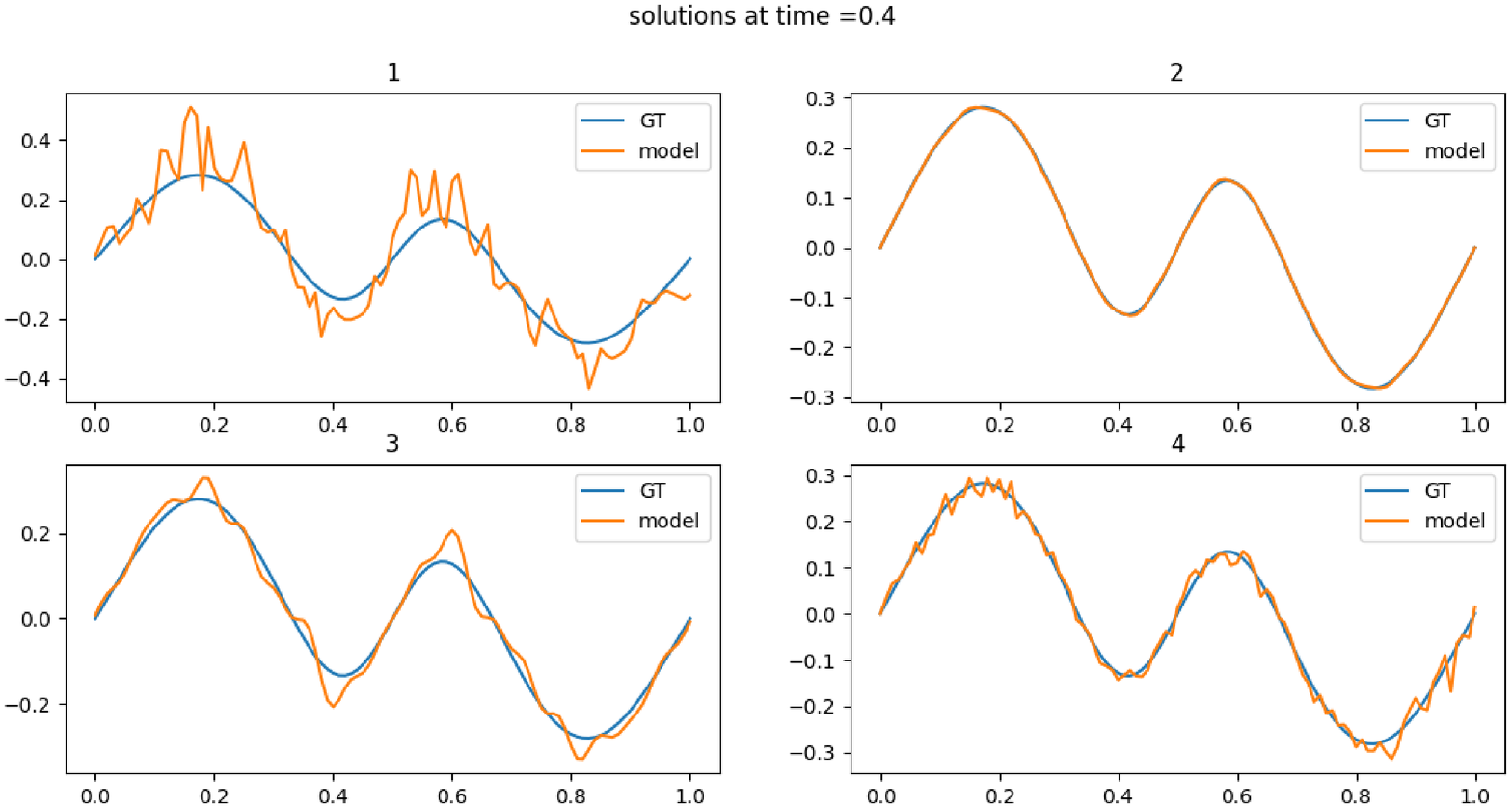} 
\par\end{centering}
\caption{\label{fig:examples-of-exact segment}Heat equation over {[}0,1{]}
- comparisons of the ground truth solution and the different neural
network solutions with different initial conditions and at different
times}
\end{figure}

In addition, we performed generalization and stability analysis for
the different architectures. To evaluate the ability of our networks
to generalize beyond the training space of polynomials of degree 20,
we tested the different networks using initial conditions from a space
of polynomials of degree 30. Namely, 
\[
W_{30}=\left\{ \sum_{k=1}^{30}c_{k}\sin(2\pi kx),\quad c_{1},...,c_{20}\in[-1,1],\sqrt{c_{1}^{2}+...+c_{20}^{2}}=1\right\} .
\]
To evaluate the stability of our networks, we added normal random
noise with mean $0$ and variance $0.3$ to the initial condition
sample vectors and evaluated at different time stamps the following
normalized metric 
\begin{equation}
\frac{\|\tilde{u}(\vec{f}+\vec{\delta},\cdot,t)-\tilde{u}(\vec{f},\cdot,t)\|_{2}}{\|\delta\|_{2}},\label{norm-metric}
\end{equation}
where $\delta_{i}\sim N(0,0.3)$. The results of the generalization
test can be found in Table \ref{tab:Heat-equation-segment-generalization},
and the results, averaged over 20 random initial conditions, for the
stability test can be found in Table \ref{tab:Heat-equation-segment-stability}.
In both tests, we can observe that all spectral model variants outperform
the naive model. 
\begin{table}
\centering{}%
\begin{tabular}{|>{\centering}p{1.5cm}|>{\centering}p{6cm}|c|}
\hline 
Model number in plots  & Model Architecture  & MSE\tabularnewline
\hline 
\hline 
1  & Naive Model  & 1.0e-3\tabularnewline
\hline 
2  & Spectral model - full realization (time stepping and reconstruction
blocks)  & \textbf{7.1e-4}\tabularnewline
\hline 
3  & Spectral model - MLP approximation of time stepping block, realization
of reconstruction block  & 8.1e-4\tabularnewline
\hline 
4  & Spectral model - realization of time stepping block, MLP approximation
of reconstruction block  & 7.4e-4\tabularnewline
\hline 
\end{tabular}\caption{\label{tab:Heat-equation-segment-generalization}Heat equation on
$[0,1]$ - generalization results }
\end{table}

\begin{table}
\begin{centering}
\begin{tabular}{|>{\centering}p{1.5cm}|>{\centering}p{5cm}|c|c|c|}
\hline 
Model number in plots  & Model Architecture  & $T=0.2$  & $T=0.4$  & $T=0.5$\tabularnewline
\hline 
\hline 
1  & Naive Model  & 3.68  & 3.82  & 3.77\tabularnewline
\hline 
2  & Spectral model - full realization (time stepping and reconstruction
blocks)  & \textbf{{1.07} } & \textbf{{0.84} } & \textbf{{0.78}}\tabularnewline
\hline 
3  & Spectral model - MLP approximation of time stepping block, realization
of reconstruction block  & 1.2  & 0.99  & 0.96\tabularnewline
\hline 
4  & Spectral model - realization of time stepping block, MLP approximation
of reconstruction block  & \textbf{{1.07} } & 0.85  & \textbf{{0.78}}\tabularnewline
\hline 
\end{tabular}
\par\end{centering}
\caption{\label{tab:Heat-equation-segment-stability}Heat equation on $[0,1]$
- stability test results using the normalized metric \eqref{norm-metric}}
and noise $\sim N(0,0.3)$. 
\end{table}

The theoretical and empirical results for the simple case of the heat
equation over $\Omega=[0,1]$ motivate us to establish guidelines
for designing spectral PINN networks in much more complicated scenarios.
Namely, we should try to realize the various blocks, approximate them
or at the least design them inspired by the spectral method.

\section{The sphere $\mathbb{S}^{2}$}

In this section, we demonstrate our method in a more challenging setup,
a nonlinear equation on a curved manifold. The Allen-Cahn equation
over the sphere $\mathbb{S}^{2}$ is defined by \cite{key-14}: 
\begin{equation}
u_{t}=\epsilon\Delta u+u-u^{3},\label{Allen-Cahn-sphere}
\end{equation}
where $\epsilon>0$ and the Laplace-Beltrami operator is 
\[
\Delta=\frac{\partial^{2}}{\partial\theta^{2}}+\frac{\cos\theta}{\sin\theta}\frac{\partial}{\partial\theta}+\frac{1}{\sin^{2}\theta}\frac{\partial^{2}}{\partial\phi^{2}},
\]
with $\phi$ is the azimuth angle and $\theta$ is the polar angle.

\subsection{Theory and spectral PINN architecture for the Allen-Cahn equation
on $\mathbb{S}^{2}$}

On $\mathbb{S}^{2}\subset\mathbb{R}^{3}$ the spectral basis is the
spherical harmonic functions \cite{key-9}: 
\begin{defn}
\label{def:The-spherical-harmonic}The spherical harmonic function
of degree $l$ and order $m$ is given by: 
\[
Y_{l}^{m}(\theta,\phi)=(-1)^{m}\sqrt{\frac{(2l+1)}{4\pi}\frac{(l-m)!}{(l+m)!}}P_{l}^{m}(\cos\theta)e^{im\phi},
\]
where $\theta\in[0,\pi]$ is the polar angle, $\phi\in[0,2\pi)$ is
the azimuth angle and $P_{l}^{m}:[-1,1]\rightarrow\mathbb{R}$ is
the associated Legendre polynomial. 
\end{defn}

Each spherical harmonic function is an eigenfunction of the Laplace-Beltrami
operator satisfying 
\[
\Delta Y_{l}^{m}=-l(l+1)Y_{l}^{m}.
\]
In our work, for simplicity, we use the real version of the spherical
harmonics, defined by: 
\[
Y_{lm}=\begin{cases}
\sqrt{2}(-1)^{m}Im(Y_{l}^{|m|}), & -l\le m<0,\\
Y_{l}^{0}, & m=0,\\
\sqrt{2}(-1)^{m}Re(Y_{l}^{m}), & 0<m\le l.
\end{cases}
\]
\\
 The inputs to our networks are of type $(F,(\theta,\phi),t)$, where
$F\in\mathbb{R}^{20\times20}$ is a sampling matrix of the initial
condition on uniform azimuth-polar grid of a spherical function, $\theta\in[0,\pi],\phi\in[0,2\pi)$
are the coordinates of a point on the sphere and $t\in[0,1]$. The
loss functions are similar to the loss functions used in section 4,
with the required modifications, such as for the differential loss
term 
\begin{equation}
L_{D}(\theta)=\frac{1}{N}\sum_{i=1}^{N}\left|\frac{\partial\tilde{u}_{\theta}(F_{i},x_{i},t_{i})}{\partial t}-(\epsilon\Delta\tilde{u}_{\theta}+\tilde{u}_{\theta}-\tilde{u}_{\theta}^{3})(F_{i},x_{i},t_{i})\right|^{2}.\label{L-D-loss-sphere}
\end{equation}
Our goal is to construct a spectral PINN architecture that will outperform
the naive PINN architecture. Here are the details of the 3 blocks
of the spectral model that follow the blueprint of Section \ref{sec:spectralPINN}
: 
\begin{enumerate}
\item \textbf{Transformation Block} \\
 This block receives as input a flatten sampling matrix $\vec{F}\in\mathbb{R}^{400}$
of an initial condition $f$ from the space 
\[
\sum_{l=0}^{9}\sum_{m=-l}^{l}c_{l,m}Y_{lm}(\theta,\phi).
\]
It returns the 100 spherical harmonic coefficients of degree 9. By
\cite[Theorem 3]{key-18} under these conditions, spherical harmonics
of degree 9 can be perfectly reconstructed. Thus, training one dense
linear layer $\mathcal{\tilde{C}}:\mathbb{R}^{400}\rightarrow\text{\ensuremath{\mathbb{R}^{100}}}$,
recovers the perfect reconstruction formula 
\[
\tilde{\mathcal{C}}(\vec{F})=(c_{0,0},c_{1,-1},c_{1,0},c_{1,1},...,c_{9,-9},...,c_{9,0},...,c_{9,9}),
\]
\item \textbf{Time Stepping Block}\\
 Unlike the heat equation on the unit interval, the Allen-Cahn equation
\eqref{Allen-Cahn-sphere} on the sphere, does not admit an analytic
spectral solution. Nevertheless, we design an architecture that follows
the spectral paradigm and compare it with a standard PINN MLP architecture.
We test our hypothesis by conducting an ablation study using three
optional architectures for the time stepping block: 
\begin{enumerate}
\item \textbf{Input of Allen-Cahn Nonlinear Part}\\
 In this architecture, we further adapt the architecture to the nature
of the equation, specifically to the non-linear part of the Allen-Cahn
equation. Thus, in this variant, the input to the time stepping block
is composed of: the transformation of the initial condition, the transformation
of the nonlinear part of the initial condition and the time variable
$(\tilde{\mathcal{C}}(\vec{F}),\tilde{\mathcal{C}}(\vec{F}-\vec{F}^{3}),t)$.
Therefore the time stepping block is defined as 
\[
\mathcal{\tilde{D}}:\mathbb{R}^{100}\times\mathbb{R}^{100}\times[0,T]\rightarrow\mathbb{R}^{100},
\]
where 
\[
\mathcal{\tilde{D}}(\tilde{\mathcal{C}}(\vec{F}),\tilde{\mathcal{C}}(\vec{F}-\vec{F}^{3}),t)=(c_{0,0}(t),c_{1,-1}(t),c_{1,0}(t),c_{1,1}(t),...,c_{9,-9}(t),...,c_{9,0}(t),...,c_{9,9}(t)).
\]
With the additional input of the non-linear part, this variant of
the time stepping block is a sum of two sub-blocks $\mathcal{\tilde{D}}=\mathcal{\tilde{D}}_{1}+\mathcal{\tilde{D}}_{2}$.
The component $\mathcal{\tilde{D}}_{1}$ is a sub-block designed to
capture an exponential dynamic of the solution across time. The sub-block
$\mathcal{\tilde{D}}_{2}$ is a standard PINN sub-block. The exponential
sub-block $\mathcal{\tilde{D}}_{1}$ is defined by 
\[
\mathcal{\tilde{D}}_{1}(\tilde{\mathcal{C}}(\vec{F}),\tilde{\mathcal{C}}(\vec{F}-\vec{F}^{3}),t)=e^{\mathcal{\tilde{D}}_{1,1}(t)}\odot\mathcal{\tilde{D}}_{1,2}(\tilde{\mathcal{C}}(\vec{F}),\tilde{\mathcal{C}}(\vec{F}-\vec{F}^{3})),
\]
where $\odot$ is element-wise vector multiplication. The component
$\mathcal{\tilde{D}}_{1,1}:\mathbb{R}\rightarrow\mathbb{R}^{100}$
is a simple dense layer with no bias, i.e. $\mathcal{\tilde{D}}_{1,1}(t)=V\cdot t$
where $V\in\mathbb{R}^{100}$ is a learnable vector. The component
$\mathcal{\tilde{D}}_{1,2}:\mathbb{R}^{100}\times\mathbb{R}^{100}\rightarrow\mathbb{R}^{100}$
is an MLP subnetwork with 6 layers with $\tanh$ activations. Finally,
the sub-block $\mathcal{\tilde{D}}_{2}$ is also an MLP subnetwork
with 6 layers and $\tanh$ activations. The full architecture with
this time stepping variant is depicted in Figure \ref{fig:Full-architecture-spherical}.
\begin{figure}
\begin{centering}
\includegraphics[width=12cm,height=8cm]{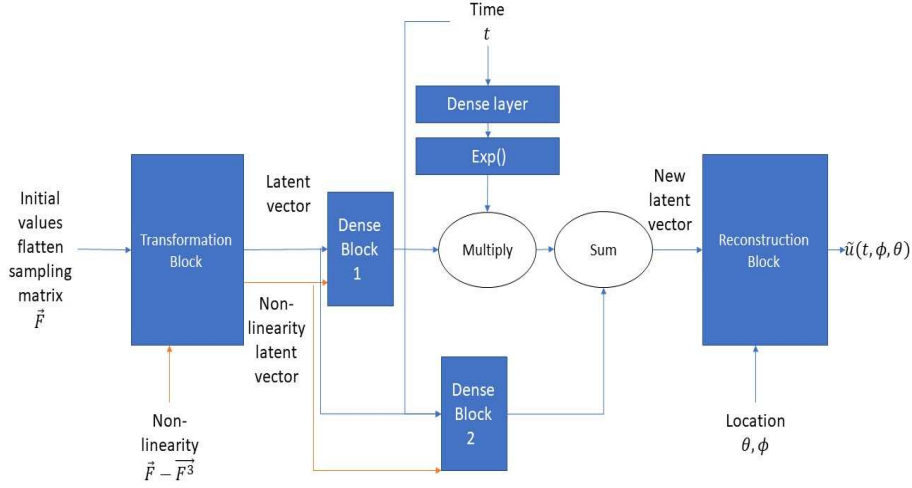} 
\par\end{centering}
\caption{\label{fig:Full-architecture-spherical}Full architecture for spherical
setting - the red arrows are used only in variant (a) for time stepping
block}
\end{figure}
\item \textbf{Standard Exponential Block}\\
 This variant of the time-stepping block is similar to the one described
in (a), but without the non-linear input $\tilde{\mathcal{C}}(\vec{F}-\vec{F}^{3})$.
Thus, the subnetwork capturing the exponential behavior takes the
form 
\[
\mathcal{\tilde{D}}_{1}(\tilde{\mathcal{C}}(\vec{F}),t)=e^{\mathcal{\tilde{D}}_{1,1}(t)}\odot\mathcal{\tilde{D}}_{1,2}(\tilde{\mathcal{C}}(\vec{F})).
\]
In this architecture we add 5 more dense layers to this block, as
each layer requires less weights. 
\item \textbf{Naive MLP Time Stepping Block}\\
 In this variant of the time stepping block, the input is $(\tilde{\mathcal{C}}(\vec{F}),t)$
and the architecture is a simple MLP block of 12 layers with $\tanh$
activation functions. 
\end{enumerate}
\item \textbf{Reconstruction Block}\\
 The heuristics of our spectral approach is that the output of the
time stepping block should be (once trained) a representation space
resembling the coefficients of the spectral basis at the given time.
Therefore, we design the reconstruction block to be composed of dense
layers, but we use activation functions of the form $\sin^{l},\cos^{l}$,
$0\leq l\leq9$, on the input data point $(\theta,\phi)$, since these
activation functions are the building blocks of the spherical harmonics
functions. To this end, we first apply two subnetworks on the data
point $(\theta,\phi)$ 
\[
\mathcal{R}_{l,\sin,0}(\theta,\phi),\mathcal{R}_{l,\cos,0}(\theta,\phi):\mathbb{R}^{2}\rightarrow\mathbb{R}^{2}.
\]
We then apply on their output, component wise, the spectral activation
functions 
\[
\sin^{l}\circ\mathcal{R}_{l,\sin,0}(\theta,\phi),\quad\cos^{l}\circ\mathcal{D}_{l,\cos,0}(\theta,\phi),\quad0\le l\le9.
\]
Next we apply dense layers on the output of the activation functions
\[
\mathcal{R}_{l,\sin,1},\mathcal{R}_{l,\cos,1}:\mathbb{R}^{2}\rightarrow\mathbb{R}^{100},\quad0\le l\le9.
\]
We assemble these pieces to produce a subnetwork $\mathcal{R}_{loc}:\mathbb{R}^{2}\rightarrow\mathbb{R}^{100}$
\[
\mathcal{R}_{loc}(\theta,\phi)=\sum_{l=0}^{9}\mathcal{R}_{l,\sin,1}(\sin^{l}\circ\mathcal{R}_{l,\sin,0}(\theta,\phi))\odot\mathcal{R}_{l,\cos,1}(\cos^{l}\circ\mathcal{R}_{l,\cos,0}(\theta,\phi)),
\]
where $\odot$ is element-wise vector multiplication.

We apply separately, on the output of the time stepping block a subnetwork
$\mathcal{R}_{d}:\mathbb{R}^{100}\rightarrow\mathbb{R}^{100}$. Finally,
our reconstruction network $\tilde{\mathcal{R}}$ is a dot-product
between the outputs of $\mathcal{R}_{d}$ and $\mathcal{R}_{loc}$
\begin{align*}
 & \tilde{\mathcal{R}}(c_{0,0}(t),c_{1,-1}(t),c_{1,0}(t),c_{1,1}(t),...,c_{9,-9}(t),...,c_{9,0}(t),...,c_{9,9}(t),\theta,\phi)\\
 & \qquad=\langle\mathcal{R}_{d}(c_{0,0}(t),c_{1,-1}(t),c_{1,0}(t),c_{1,1}(t),...,c_{9,-9}(t),...,c_{9,0}(t),...,c_{9,9}(t)),\mathcal{R}_{loc}(\theta,\phi)\rangle.
\end{align*}

\end{enumerate}

\subsection{Experimental Results}

We generated training data consisting of $N=5,000$ randomly chosen
samples of the form $(\vec{F},(\theta,\phi),t)$, where $\vec{F}$
is a flattened sampling matrix of initial conditions randomly sampled
from 
\[
W=\left\{ {\sum_{l=0}^{9}\sum_{m=-l}^{l}c_{lm}Y_{lm}(\theta,\phi),\qquad c_{lm}\in[-1,1],\sqrt{\sum_{l=0}^{9}\sum_{m=-l}^{l}c_{lm}^{2}}=1}\right\} ,
\]
on the uniform parametric grid 
\[
\theta_{j}=\frac{\pi}{19}j,\ j\in\{0,...,19\},\qquad\phi_{k}=\frac{2\pi}{20}k,\ k\in\{0,...,19\}.
\]
During the training of the spectral model we used some manipulations
to improve the results: 
\begin{enumerate}
\item Pre-training the transformation block and the reconstruction block
separately before training the full model, using the MSE loss function
\[
\frac{1}{N}\sum_{i=1}^{N}\left|{F(\theta_{i},\phi_{i})-\tilde{\mathcal{R}}(\tilde{\mathcal{C}}(\vec{F}_{i}),(\theta_{i},\phi_{i}))}\right|^{2}.
\]
\item When training the full model, we started the first 20 epochs by freezing
the weights of the transformation and reconstruction blocks that were
pre-trained separately in (1) and training only the time stepping
block. We observed that this technique where the transformation block
and the reconstruction are pre-trained and then kept constant for
the first epochs provides better initialization of the time-stepping
block and overall better results.

In this stage of the training, we used a loss function containing
three terms. In addition to the standard initial condition loss and
the differential loss we added new loss to enforce that the time stepping
block does not change the spherical harmonics coefficients at time
zero. Formally, the new loss term over the training set is 
\begin{equation}
\frac{1}{100N}\sum_{i=1}^{N}\|\mathcal{\tilde{D}}(\tilde{\mathcal{C}}(\vec{F_{i}}),0)-\tilde{\mathcal{C}}(\vec{F_{i}})\|_{2}^{2}.\label{3rd-loss}
\end{equation}

\item Finally, we trained the full model with all 3 loss terms for 25 more
epochs. 
\end{enumerate}
Since there is no analytical solution for the Allen-Cahn equation
over $\mathbb{S}^{2}$, we used the numerical scheme IMEX-BDF4 \cite{key-14}
as ground truth for testing our models. Unlike \cite{key-14}, we
used the spherical harmonic functions basis and not the double spherical
Fourier method which was used in \cite{key-14} due to performance
considerations. We tested our models using 20 random initial conditions
and predicted the solutions for all grid points: 
\[
\theta_{j}=\frac{\pi}{19}j,\ j\in\{0,...,19\},\quad\phi_{k}=\frac{2\pi}{20}k,\ k\in\{0,...,19\},\quad t_{n}=\frac{1}{500}n,\ n\in\{0,...,500\}.
\]
We benchmarked 3 spectral PINN variants of the with the naive PINN
model that has MLP architecture consisting of 26 layers with $\tanh$
activations. Table \ref{tab:Allen-Cahn-equation-spherical} shows
the comparison of the 4 models for two cases of the diffusion coefficient
in \eqref{Allen-Cahn-sphere} $\epsilon=0.01,0.001$. As in Subsection
\ref{subsec:heat-experiment}, testing was performed by measuring
MSE for the approximated solutions for 20 initial conditions over
500 uniform time steps. We can see that our model achieves better
accuracy than the naive model, with significantly less parameters.
We can also see that there is a benefit to the special processing
of the non-linear part of Allen Cahn equation by feeding the time
stepping block with the non-linear part of the initial condition.
In Figure \ref{fig:Error-versus-time-spherical} we show the norm
of the error in different time steps for the case $\epsilon=0.1$.
As in the previous example, we performed generalization and stability
tests. For the generalization test we used random initial conditions
from the larger set of spherical harmonics of degree 14: 
\[
W_{G}=\left\{ {\sum_{l=0}^{14}\sum_{m=-l}^{l}c_{lm}Y_{lm}(\theta,\phi),\qquad c_{lm}\in[-1,1],\sqrt{\sum_{l=0}^{14}\sum_{m=-l}^{l}c_{lm}^{2}}=1}\right\} .
\]
For the stability test we used the technique as in the previous section
with noise $\delta\sim N(0,0.3)$ and the metric \eqref{norm-metric}.
The results of generalization and stability tests can be found in
tables \ref{tab:sphre - generalization} and \ref{tab: sphere-stability}
respectively (averaged over 20 random initial conditions). Again,
we can see that all spectral model variants outperform the naive model.

\begin{table}
\begin{centering}
\begin{tabular}{|>{\centering}p{1.5cm}|>{\centering}p{5cm}|c|c|c|}
\hline 
Model number in plots  & Model Architecture  & \#weights  & MSE with $\epsilon=0.1$  & MSE with $\epsilon=0.001$ \tabularnewline
\hline 
\hline 
1  & Naive Model  & 4,070,704  & 1.1e-4  & 2.1e-4 \tabularnewline
\hline 
2  & Spectral model, time stepping variant (a) - Input of Allen-Cahn nonlinear
part  & \textbf{391,186}  & \textbf{4.8e-5}  & \textbf{6.1e-5} \tabularnewline
\hline 
3  & Spectral model, time stepping variant (b) - Standard time stepping
exponential block  & 452,590  & 6.7e-5  & 6.9e-5 \tabularnewline
\hline 
4  & Spectral model, time stepping variant (c) - Naive time stepping dense
block  & 490,682  & 9.7e-5  & 8.1e-5 \tabularnewline
\hline 
\end{tabular}
\par\end{centering}
\caption{\label{tab:Allen-Cahn-equation-spherical}Allen-Cahn equation \eqref{Allen-Cahn-sphere}
over $\mathbb{S}^{2}$ with $\epsilon=0.1,0.001$ - Comparison of
standard naive PINN model with 3 variants of our spherical PINN model}
\end{table}

\begin{figure}
\centering{}\includegraphics[width=12cm,height=8cm]{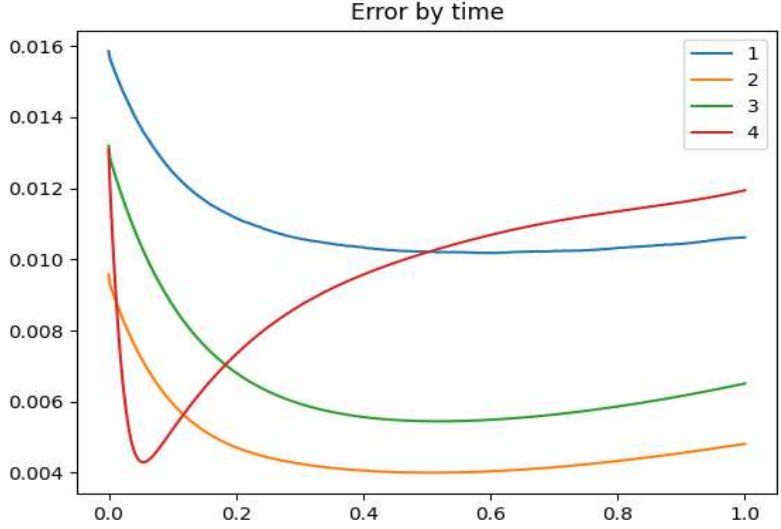}\caption{\label{fig:Error-versus-time-spherical} Allen-Cahn equation over
$\mathbb{S}^{2}$ with $\epsilon=0.1$ - Error over time of the naive
and spectral variant PINN models on testing dataset}
\end{figure}

\begin{table}
\begin{centering}
\begin{tabular}{|>{\centering}p{1.5cm}|>{\centering}p{6cm}|c|}
\hline 
Model number in plots  & Model Architecture  & MSE\tabularnewline
\hline 
\hline 
1  & Naive Model  & 3.6e-4\tabularnewline
\hline 
2  & Spectral model, time stepping variant (a) - Input of Allen-Cahn nonlinear
part  & \textbf{1.1e-4}\tabularnewline
\hline 
3  & Spectral model, time stepping variant (b) - Standard time stepping
exponential block  & 1.3e-4\tabularnewline
\hline 
4  & Spectral model, time stepping variant (c) - Naive time stepping dense
block  & 1.2e-4\tabularnewline
\hline 
\end{tabular}
\par\end{centering}
\caption{\label{tab:sphre - generalization}Allen-Cahn equation over $\mathbb{S}^{2}$
with $\epsilon=0.1$ - generalization test results}
\end{table}

\begin{table}
\centering{}%
\begin{tabular}{|>{\centering}p{1.5cm}|>{\centering}p{6cm}|c|c|c|}
\hline 
Model number in plots  & Model Architecture  & $T=0.4$  & $T=0.7$  & $T=1.0$\tabularnewline
\hline 
\hline 
1  & Naive Model  & 3.3  & 3.29  & 3.28\tabularnewline
\hline 
2  & Spectral model, time stepping variant (a) - Input of Allen-Cahn nonlinear
part  & \textbf{0.69}  & \textbf{0.65}  & \textbf{0.65} \tabularnewline
\hline 
3  & Spectral model, time stepping variant (b) - Standard time stepping
exponential block  & 0.79  & 0.74  & 0.73 \tabularnewline
\hline 
4  & Spectral model, time stepping variant (c) - Naive time stepping dense
block  & 2.9  & 2.8  & 2.8\tabularnewline
\hline 
\end{tabular}\caption{\label{tab: sphere-stability}Allen-Cahn equation over $\mathbb{S}^{2}$
with $\epsilon=0.1$ - stability test results using the normalized
metric \eqref{norm-metric} and noise $\sim N(0,0.3)$.}
\end{table}

Next we compare the training time required for the models. In Figure
\ref{fig:sphere-train-time} we see the training loss over the training
epochs for the naive PINN model 1 and the spectral PINN variant 2.
The left hand zoom out plot takes into account the training epochs
used by the spectral PINN for the initial training of the transformation-reconstruction
subnetworks and then plots the MSE for the training of the full spectral
network. The right hand side shows the MSE at finer resolution over
the last epochs.

Lastly, we compared our method to a classic PINN approach that trains
specific 20 different neural networks for each of the separate 20
test initial conditions (see Subsection \ref{subsec:PINN}). These
networks receive as input a location on the sphere and time step and
provide as output an approximate solution for only the unique initial
condition they trained for. Each separate network has 1,467,324 weights
and is constructed using 10 dense inner layers. The average MSE of
the 20 networks was 4.9e-5 which is comparable to the MSE of our spectral
method provided in Table \ref{tab:Allen-Cahn-equation-spherical}.
Yet, our spectral network model is smaller and provides approximations
for any initial condition from the set $W$ with no additional training.
On the other hand, the classic PINNs can be trained for any initial
condition under weaker assumptions. The average training time for
each separate classic vanilla PINN was 42 minutes while for the spectral
PINN that can take as input any initial condition from $W$ it was
103 minutes.

\begin{figure}
\centering{}\includegraphics[width=8cm]{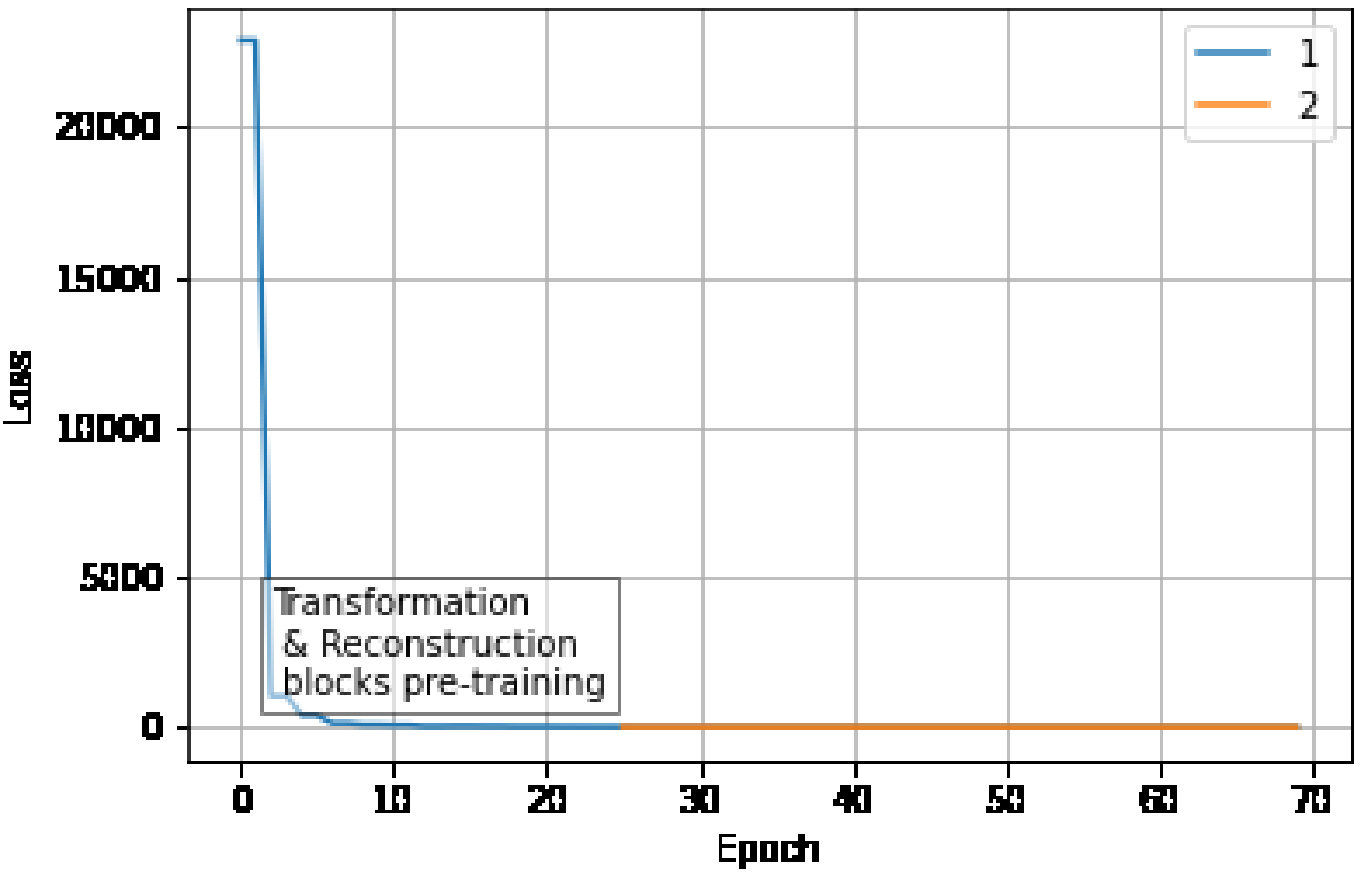}\includegraphics[width=8cm]{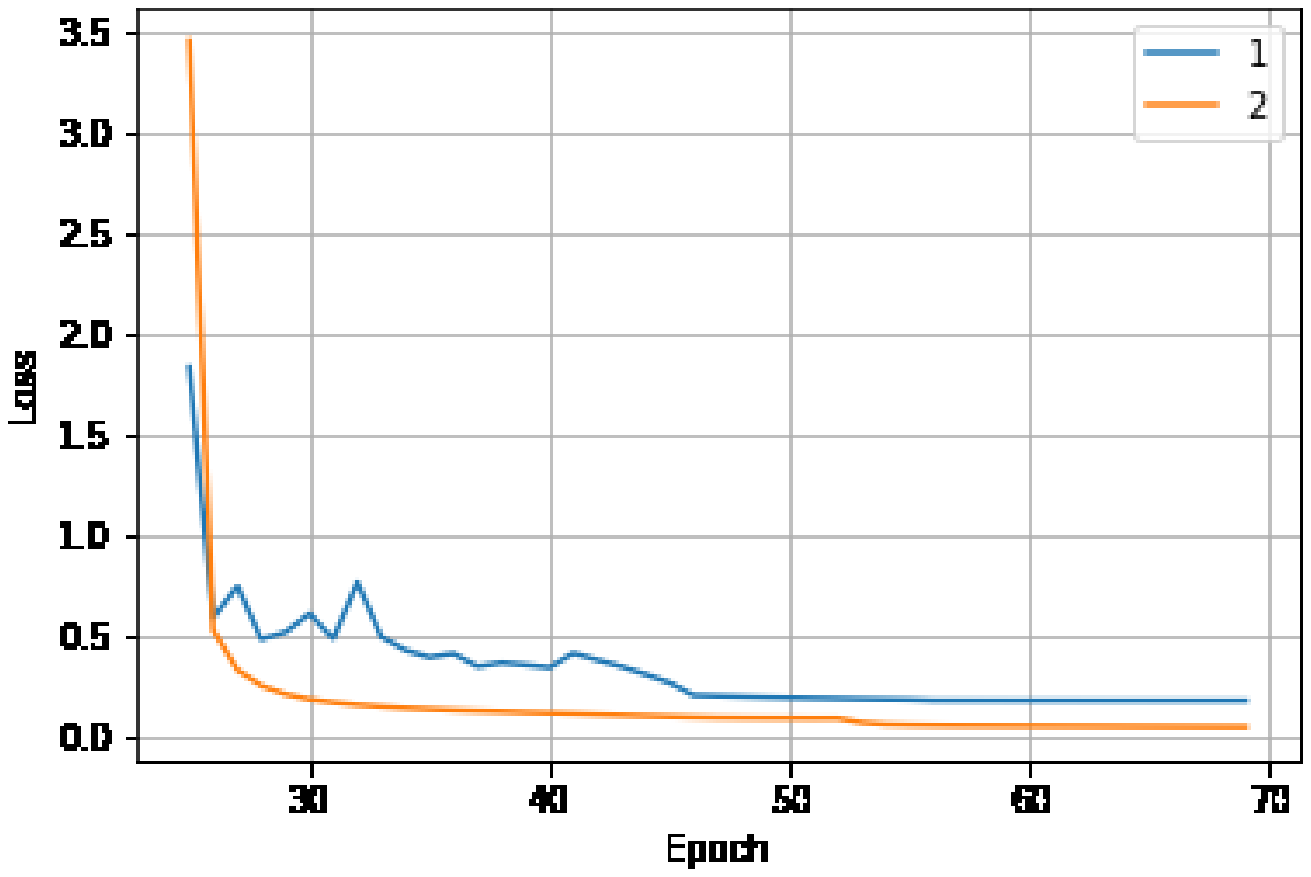}
\caption{\label{fig:sphere-train-time} Allen-Cahn equation over $\mathbb{S}^{2}$
with $\epsilon=0.1$ - Comparison of training loss over epochs}
\end{figure}

\section{The embedded torus $\mathbb{T}\subset\mathbb{R}^{3}$}

\label{sec:torus}

In this section, we demonstrate our method on the embedded torus 
\[
\mathbb{T}=\{((R+r\cos\theta)\cos\phi,(R+r\cos\theta)\sin\phi,r\sin\theta)|\theta,\phi\in[0,2\pi)\}\subset\mathbb{R}^{3}.
\]
In this setting, the Laplace-Beltrami operator is \cite{key-10} 
\[
\Delta_{\mathbb{T}}=\frac{1}{r^{2}}\frac{\partial^{2}}{\partial\theta^{2}}-\frac{\sin\theta}{r(R+r\cos\theta)}\frac{\partial^{2}}{\partial\theta}+\frac{1}{(R+r\cos\theta)^{2}}\frac{\partial^{2}}{\partial\phi^{2}}.
\]
On this manifold, there is no analytic form of the spectral basis
and so coefficients need to be approximated from given samples of
a function as we shall see below. In Figure \ref{fig:torus-basis}
we see a rendering taken from \cite{torus-fig} of the approximations
of some of the first elements of the spectral basis on the torus.

\begin{figure}
\centering{}\includegraphics[width=10cm]{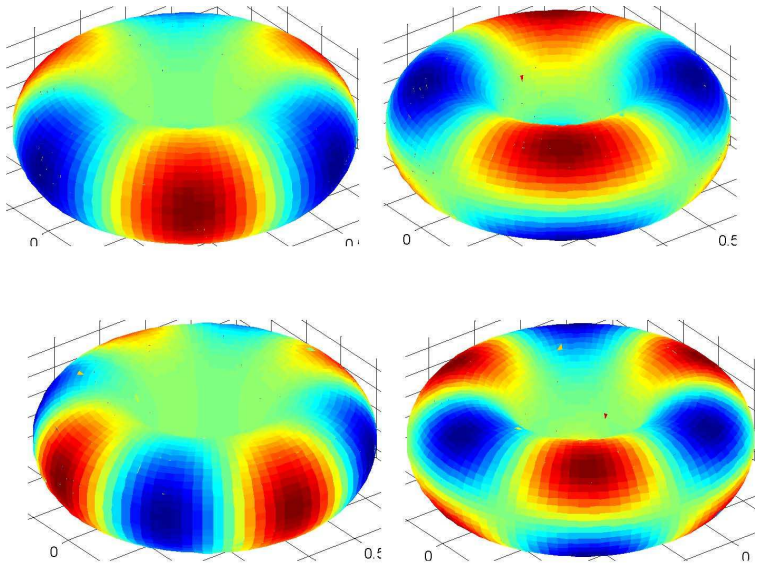}
\caption{\label{fig:torus-basis} Approximation of some of the first elements
of the spectral basis on the torus \cite{torus-fig}}
\end{figure}

On the torus we demonstrate our spectral PINN method again using the
Allen-Cahn equation \eqref{Allen-Cahn-sphere} with $\epsilon=0.1$.
As the class of initial conditions we use the set 
\[
W=\left\{ {\sum_{k=1}^{5}\sum_{l=1}^{5}c_{k,l}\sin(k\theta)\sin(l\phi),\quad\sqrt{\sum_{k.l=1}^{5}c_{k,l}^{2}}=1}\right\} .
\]
Note that in this case, the subset of initial conditions is not a
subspace of the manifold's spectral basis. We sample functions from
this set on a uniform parametric grid with $N_{\theta}=N_{\phi}=15$.
On the embedded torus one is required to use a numeric approximation
of the spectral basis and we used the finite-elements method implemented
in the python package SPHARAPY \cite{key-11}. The choice of spectral
basis implementation impacts the design of the architecture of the
transformation and reconstruction blocks. We test several options
for each block. For the transformation and reconstruction blocks we
consider two options: 
\begin{enumerate}
\item \textbf{Numerical Spectral basis blocks -} In this option, we first
create a dataset of 5,000 triples, each composed of a sampling matrix
of a function $f\in W$, on a uniform parametric grid with $N_{\theta}=N_{\phi}=15$
and two random coordinates $(\theta,\phi)\in[0,2\pi)^{2}$ of a point
on the torus. We then train the transformation and reconstruction
blocks separately as follows.

For the training of the transformation block $\mathcal{\tilde{C}}$
we further approximate for each function in the training set, using
its sampling matrix, the (numerical) spectral transformation using
SPHARAPY. The package numerically computes for each set of samples,
the first $K$ coefficients of the spectral basis. Thus, we applied
SPHARAPY with $K=225$ and used its output as ground truth to train
our transformation block. The block's architecture is composed of
3 convolution layers followed by one dense layer.

The reconstruction block $\mathcal{\tilde{R}}$ in this variant is
trained to take as input the coefficients of the spectral representation
and the coordinate $(\theta,\phi)\in[0,2\pi)^{2}$ and approximate
the ground truth function value at this coordinate. The block architecture
is a MLP subnet with 15 layers. 
\item \textbf{Auto-Encoder-Decoder blocks -} Auto-encoder-decoder architectures
are very popular in deep learning applications \cite{BKG}. Their
goal is to learn compact representation spaces of data. This is achieved
through two networks that are trained simultaneously. The encoder
network takes the input space of dimension $M$ and applies a nonlinear
transformation using several layers into a smaller representation
space of dimension $K<M$. The decoder network then takes the compressed
representation and trains to approximately recover the original $M$-th
dimensional data or a certain piece of information relating to the
original data. As we shall see, our application is the latter.

The motivation to use the concept of an encoder-decoder architecture
in our setting is to provide an alternative to the the complexity
of using numerical approximations of the spectral basis over manifolds,
by learning an alternative useful non-linear transformation into a
compact representation. Thus, we train a transformation block $\mathcal{\tilde{C}}$
as the encoder together with the reconstruction block $\mathcal{\tilde{R}}$
as a decoder, without using explicitly the spectral representation
on the torus. However, this approach is certainly inspired by the
spectral method as we are ultimately optimizing some compressed representation
space. The transformation encoder block simply learns to create a
compressed latent representation of dimension $K=150$ from the $M=225$
function samples in a representation space. The encoder's architecture
is composed of 5 convolution layers followed by one dense layer. Then
the decoder takes the compressed representation in dimension $K=150$
together with a coordinate $(\theta,\phi)\in[0,2\pi)^{2}$ and tries
to recover the ground truth function value at this coordinate. Its
architecture is 17 dense layers. The loss function over the training
set is then 
\[
\frac{1}{N}\sum_{i=1}^{N}\left|{\mathcal{\tilde{R}}(\mathcal{\tilde{C}}(\vec{F}_{i}),(\theta_{i},\phi_{i}))-f_{i}(\theta_{i},\phi_{i})}\right|^{2}.
\]

\end{enumerate}
For the time stepping block we test two options 
\begin{enumerate}
\item[(a)] A custom made time stepping block that receives as input the coefficients
of the initial condition as well as the coefficients of the nonlinear
part and a time step (similar to variant (a) of the time stepping
block in the spherical case from previous section). 
\[
(\tilde{\mathcal{C}}(\vec{F}),\tilde{\mathcal{C}}(\vec{F}-\vec{F}^{3}),t).
\]
Recall that such an architecture aims to be `more' physics aware and
adapted to the nature of the equation. For this variant of the time
stepping block we use 9 dense layers. 
\item[(b)] A network that takes as input 
\[
(\tilde{\mathcal{C}}(\vec{F}),t),
\]
without the nonlinear part. Here we used 15 dense layers. 
\end{enumerate}
We denote this block as earlier with $\tilde{\mathcal{D}}$. For testing
of our models, we used the IMEX-BDF4 numeric solver \cite{key-14}
to obtain approximations of solutions to the equations that we considered
as ground truth. In table \ref{tab:Different-model-results-torus}
we summarize the benchmarks of the various architectures and also
compare them to a naive PINN architecture, with 26 layers, that simply
takes in the samples of the initial condition as well as the time
step and location on the torus and outputs an approximation of the
value of the solution. 
\begin{table}
\begin{tabular}{|>{\centering}p{3cm}|>{\centering}p{4cm}|>{\centering}p{4cm}|c|c|}
\hline 
Model number in plots  & Transformation and Reconstruction blocks  & Time stepping block  & \#Weights  & MSE\tabularnewline
\hline 
1  & Naive Model  &  & 4,130,001  & 2.7e-4\tabularnewline
\hline 
2  & Numerical spectral basis blocks  & Spectral model, time stepping variant (a) - Input of Allen-Cahn nonlinear
part  & \textbf{2,564,105}  & \textbf{2.5e-5}\tabularnewline
\hline 
3  & Auto-encoder-decoder blocks  & Spectral model, time stepping variant (a) - Input of Allen-Cahn nonlinear
part  & 3,800,555  & 8.3e-5\tabularnewline
\hline 
4  & Numerical spectral basis blocks  & Spectral model, time stepping variant (b)  & 3,129,976  & 1.7e-4\tabularnewline
\hline 
\end{tabular}

\caption{\label{tab:Different-model-results-torus}Allen-Cahn equation over
$\mathbb{T}\subset\mathbb{R}^{3}$ - Comparison of standard naive
PINN model with 3 variants of our spherical PINN model}
\end{table}

We can observe that the best result, in terms of accuracy and smaller
size of the network, can be obtained using both numerical spectral
basis blocks as transformation and reconstruction blocks, combined
with the non-linear input time stepping block. Also, even the encoder-decoder
variant that `follows' the spectral paradigm to some extent without
actually using the numerical spectral basis, provides a better result
than the naive PINN model. In Figure \ref{fig:Error-versus-time-torus}
we show time plots of errors of the different PINN models averaged
over 20 random initial conditions. For the generalization test presented
in Table \ref{tab:torus-generalization}, the network that was trained
on samples from $W$ was tested on random initial conditions from
the larger set 
\[
W_{G}=\left\{ {\sum_{k=1}^{10}\sum_{l=1}^{10}c_{k,l}\sin(k\theta)\sin(l\phi),\quad\sqrt{\sum_{k.l=1}^{10}c_{k,l}^{2}}=1}\right\} .
\]
The stability tests listed in Table \ref{tab:torus- stability} are
averaged over 20 random initial conditions.

\begin{figure}
\begin{centering}
\includegraphics[width=12cm,height=8cm]{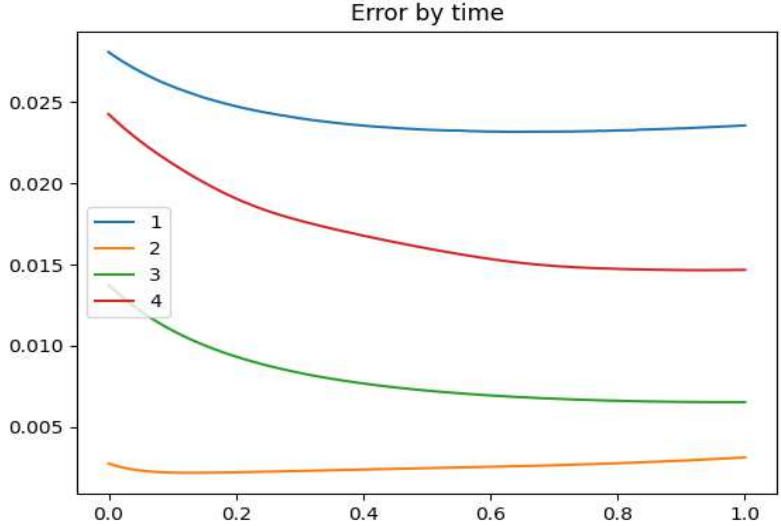} 
\par\end{centering}
\caption{\label{fig:Error-versus-time-torus}Allen-Cahn equation over $\mathbb{T}\subset\mathbb{R}^{3}$
- Error over time of the naive and spectral variant PINN models on
testing dataset}
\end{figure}

\begin{table}
\centering{}%
\begin{tabular}{|>{\centering}p{1.5cm}|>{\centering}p{6cm}|>{\centering}p{4cm}|c|}
\hline 
Model number

in plots  & Transformation and Reconstruction blocks  & Time stepping block  & MSE\tabularnewline
\hline 
1  & \multicolumn{2}{c|}{Naive Model} & 2.1e-4 \tabularnewline
\hline 
2  & Numerical spectral basis blocks  & Spectral model, time stepping variant (a) - Input of Allen-Cahn nonlinear
part  & \textbf{9.7e-5}\tabularnewline
\hline 
3  & Auto-encoder-decoder blocks  & Spectral model, time stepping variant (a) - Input of Allen-Cahn nonlinear
part  & 1.1e-4\tabularnewline
\hline 
4  & Numerical spectral basis blocks  & Spectral model, time stepping variant (b)  & 2.0e-4\tabularnewline
\hline 
\end{tabular}\caption{\label{tab:torus-generalization}Allen-Cahn equation over $\mathbb{T}\subset\mathbb{R}^{3}$
- generalization test results}
\end{table}

\begin{table}
\begin{tabular}{|>{\centering}p{1.5cm}|>{\centering}p{4cm}|>{\centering}p{4cm}|c|c|c|}
\hline 
Model number in plots  & Transformation and Reconstruction blocks  & Time stepping block  & $T=0.4$  & $T=0.7$  & $T=1.0$\tabularnewline
\hline 
1  & \multicolumn{2}{c|}{Naive Model} & 2.4  & 2.4  & 2.5 \tabularnewline
\hline 
2  & Numerical spectral basis blocks  & Spectral model, time stepping variant (a) - Input of Allen-Cahn nonlinear
part  & 0.66  & 0.60  & 0.63\tabularnewline
\hline 
3  & Auto-encoder-decoder blocks  & Spectral model, time stepping variant (a) - Input of Allen-Cahn nonlinear
part  & \textbf{{0.21} } & \textbf{{0.21} } & \textbf{{0.22}}\tabularnewline
\hline 
4  & Numerical spectral basis blocks  & Spectral model, time stepping variant (b)  & 0.91  & 0.83  & 0.79\tabularnewline
\hline 
\end{tabular}\caption{\label{tab:torus- stability}Allen-Cahn equation over $\mathbb{T}\subset\mathbb{R}^{3}$
- stability test results using the normalized metric \eqref{norm-metric}
and noise $\sim N(0,0.3)$}
\end{table}

Next we compare the training time required for the models. In Figure
\ref{fig:torus-train-time} we see the training loss over the training
epochs for the naive PINN model 1 and the spectral PINN variant 2.
The left hand zoom out plot takes into account the training epochs
used by the spectral PINN for the initial training of the transformation-reconstruction
subnetworks and then plots the MSE for the training of the full spectral
network. The right hand side shows the MSE at finer resolution over
the last epochs.

\begin{figure}
\centering{}\includegraphics[width=8cm]{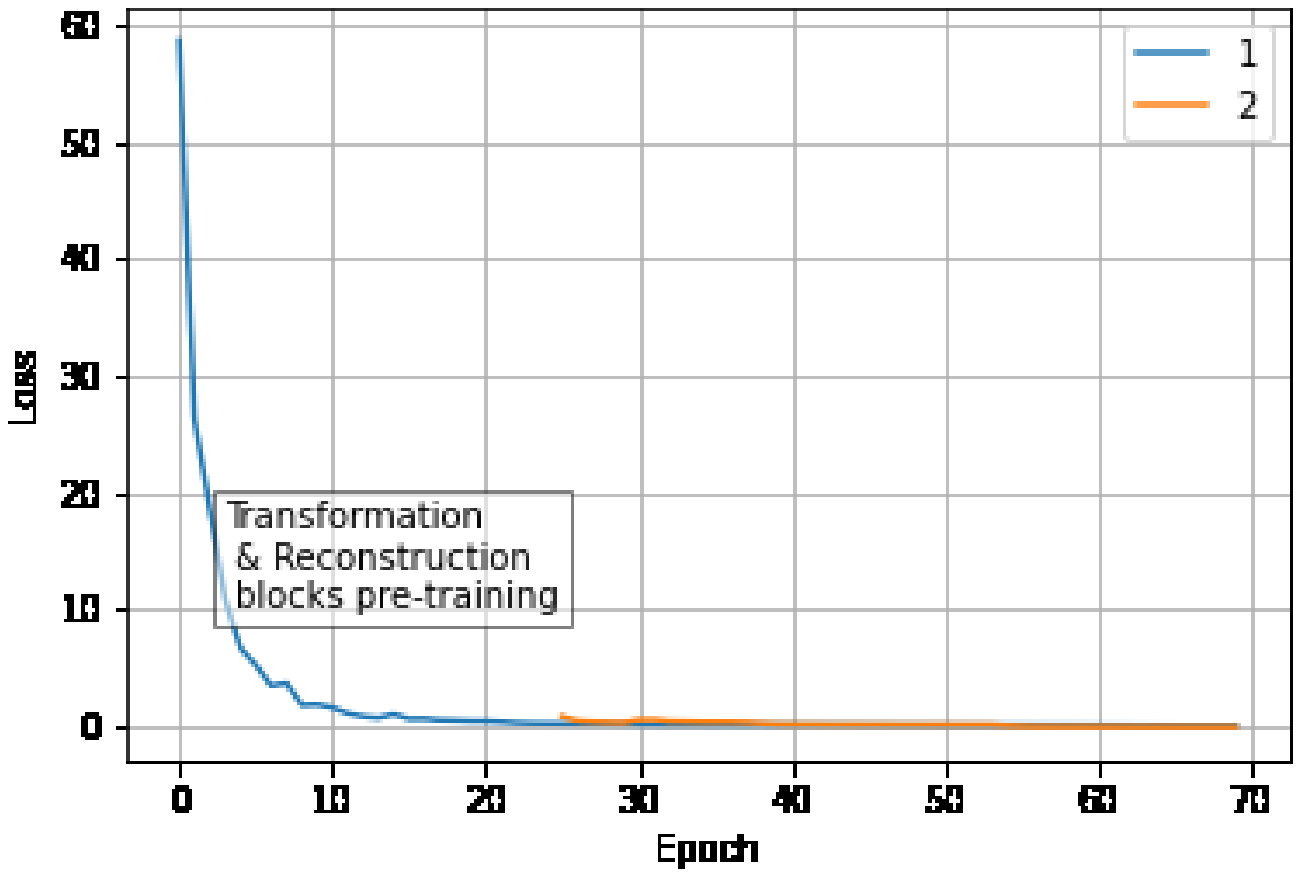}\includegraphics[width=8cm]{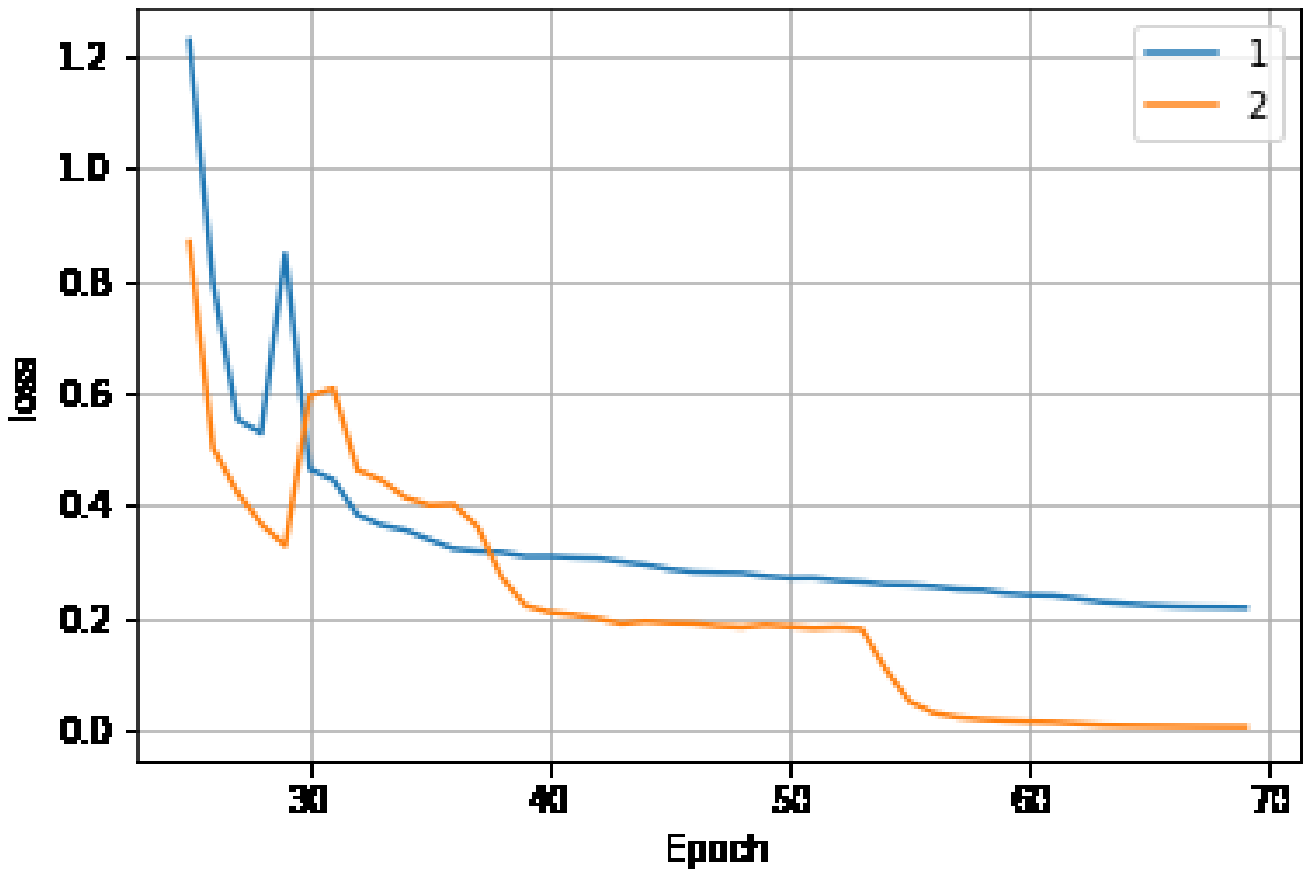}
\caption{\label{fig:torus-train-time} Allen-Cahn equation over $\mathbb{T}\subset\mathbb{R}^{3}$
- Comparison of training loss over epochs.}
\end{figure}

Finally, we tested the impact the spectral dimension $K$ has on the
accuracy and the training time. In Figure \ref{fig:SpecDimensionCompare}
we see how our spectral model (variant 2) improves with higher spectral
degrees as the training time increases. 
\begin{figure}
\begin{centering}
\includegraphics[width=10cm]{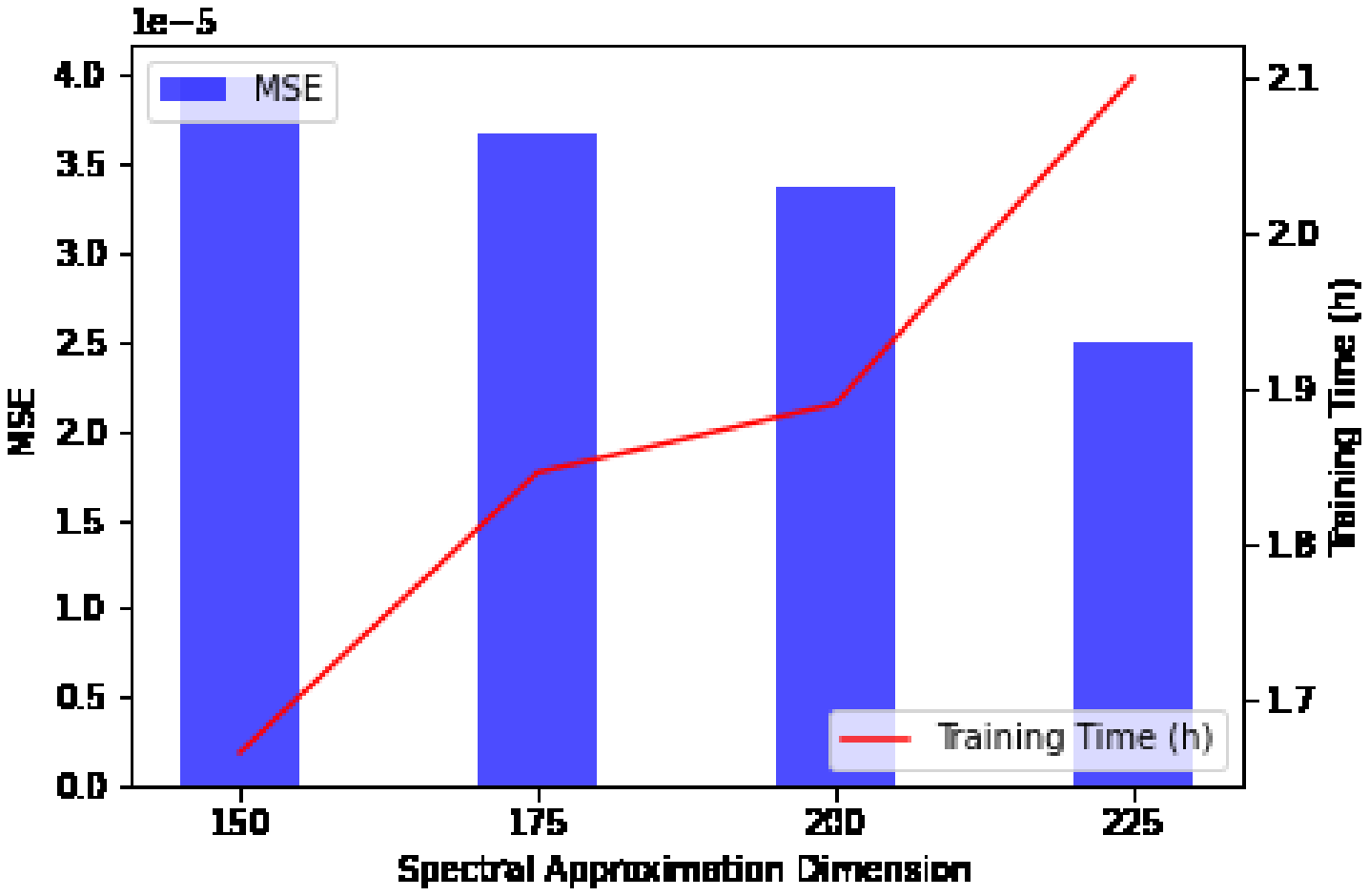} 
\par\end{centering}
\caption{\label{fig:SpecDimensionCompare}Allen-Cahn equation over $\mathbb{T}\subset\mathbb{R}^{3}$
- Comparisons of accuracy error and training time with different spectral
dimension $K$.}
\end{figure}

\section{Conclusions and future work}

In this work we presented a physics informed deep learning strategy
for building PDE solvers over manifolds which is aligned with the
method of spectral approximation. Our method allows to train a model
that can take as input initial conditions from a pre-determined subset
or subspace and is grid free. Our PI networks are designed to be aligned
with the powerful spectral methods, where on each manifold we employ
the appropriate spectral basis of the Laplace-Beltrami operator, or
an alternative encoder-decoder framework that simulates the `compression'
properties of the spectral basis. Through extensive experimentation
we empirically demonstrate that our spectral PINNs provide better
approximation with much less weights compared with standard PINN architectures.
For the case of the heat equation over the unit interval we provided
a rigorous proof for the degree of approximation of a spectral PINN
based on MLP components.

We believe that the work establishes the validity of our spectral
approach for interpolation, where the models are trained to take as
input any initial condition from the given subspace, any point on
the manifold and any time step. At this point we have not designed
and tested the models for extrapolation, such that training using
time steps from $t\in[0,0.5]$ and testing on $t\in(0.5,1]$. We shall
do so in future work. \\

\noindent \textbf{ACKNOWLEDGMENTS} The authors deeply thank the reviewers
for their numerous helpful comments and suggestions that have significantly
improved the paper.

\appendix

\section{Appendix}
\begin{proof}[Proof of theorem \ref{thm:2}]
To approximate $\mathcal{D}$, we first build two types of MLP sub-networks.
The first one $M$, approximates the multiplication 
\[
(x_{1},x_{2})\rightarrow x_{1}x_{2},\quad x_{1},x_{2}\in[-1,1].
\]
We also assume we have MLP sub-networks $E_{k}$, $1\le k\le K$,
that approximate 
\[
e^{-4\pi^{2}k^{2}\alpha t},\quad t\in[0,1].
\]
Our MLP network $\mathcal{\tilde{D}}$ can then be implemented as
the following feed forward composition 
\[
\begin{pmatrix}t\\
c_{1}\\
c_{2}\\
\vdots\\
c_{K}
\end{pmatrix}\rightarrow\begin{pmatrix}c_{1}\\
\vdots\\
c_{K}\\
E_{1}(t)\\
E_{2}(t)\\
\vdots\\
E_{K}(t)
\end{pmatrix}\rightarrow\begin{pmatrix}M(c_{1},E_{1}(t))\\
M(c_{2},E_{2}(t))\\
\vdots\\
M(c_{K-1},E_{K-1}(t))\\
M(c_{K},E_{K}(t))
\end{pmatrix}.
\]
Let us assume that our sub-networks satisfy 
\begin{equation}
|M(x_{1},x_{2})-x_{1}x_{2}|\le\frac{\epsilon}{2},\quad\forall x_{1},x_{2}\in[-1,1],\label{approx-M}
\end{equation}
\begin{equation}
|E_{k}(t)-e^{-2\pi^{2}k^{2}t}|\le\frac{\epsilon}{2},\quad1\le k\le K,,t\in[0,1].\label{approx-Ek}
\end{equation}
Then, 
\begin{align*}
|M(c_{k},E_{k}(t))-c_{k}\cdot e^{-4\pi^{2}k^{2}\alpha t}| & \leq|M(c_{k},E_{k}(t))-c_{k}\cdot E_{k}(t)|+|c_{k}\cdot E_{k}(t)-c_{k}\cdot e^{-4\pi^{2}k^{2}\alpha t}|\\
 & \leq\frac{\epsilon}{2}+|c_{k}|\frac{\epsilon}{2}\\
 & \leq\epsilon.
\end{align*}
This immediately implies that 
\[
\|\tilde{\mathcal{D}}(t,c_{0},...,c_{K})-\mathcal{D}(t,c_{0},...,c_{K})\|_{\infty}\le\epsilon.
\]
It remains to construct the MLP sub-networks $M$ and $E_{k}$, $1\le k\le K$.
Our main tool is Theorem 2.3 in \cite{key-15} which provides the
following special case. Assume $f:\mathbb{C}^{d}\rightarrow\mathbb{C}$
is analytic in the poly-ellipse defined for $\rho\ge1$ 
\[
E_{\rho}:=\left\{ {(z_{1},\dots,z_{d})\in\mathbb{C}^{d}:\quad\left|{z_{j}+\sqrt{|z_{j}|^{2}-1}}\right|\le\rho,\quad1\le j\le d}\right\} .
\]
Then, for any $\rho_{1}<\rho$ there exists a constant $c(\rho_{1},\rho)>0$,
such that for any $n\ge1$ there exist coefficients $\{a_{j}\}_{j=1}^{n}$,
vectors in $\mathbb{R}^{d}$, $\{v_{j}\}_{j=1}^{n}$ and a bias $b\in\mathbb{R}$,
such that 
\begin{equation}
\left\Vert {f-\sum_{j=1}^{n}a_{j}\tanh(v_{j}\cdot+b)}\right\Vert _{L_{\infty}[-1,1]^{d}}\le c\rho_{1}^{-n^{1/d}}\max_{z\in E_{\rho}}|f(z)|.\label{Thm2-3-M}
\end{equation}
The first application of this result, for the case $f(x_{1},x_{2})=x_{1}x_{2}$,
implies that for any $n\ge1$, there exists a subnetwork of two layers
$M_{n}$, with $O(n)$ parameters, which satisfies 
\[
\max_{x_{1},x_{2}\in[-1,1]}|x_{1}x_{2}-M_{n}(x_{1},x_{2})|\le ce^{-n^{1/2}}.
\]
Setting 
\[
\frac{c}{e^{n^{1/2}}}=\frac{\epsilon}{2},
\]
implies we should choose 
\[
n=\log^{2}\frac{2c}{\epsilon}.
\]
We conclude there exists a subnetwork $M$, with $O(\log^{2}(\epsilon^{-1}))$
weights that provides the approximation \eqref{approx-M}.

Similarly, for any $1\le k\le K$, we can apply \eqref{Thm2-3-M}
for $f(t)=e^{-2\pi^{2}k^{2}\alpha t}$, to obtain the estimate for
subnetworks with $n$ parameters $E_{k,n}$ 
\[
\max_{t\in[0,1]}|e^{-2\pi^{2}k^{2}\alpha t}-E_{k,n}(t)|\le ce^{2\pi^{2}k^{2}\alpha\rho-n},\quad1\le k\le K.
\]
Thus, we may construct subnetworks $\{E_{k}\}_{k=1}^{K}$ with $O(K^{2}+\log(\epsilon^{-1}))$
weights which provide the approximation \eqref{approx-Ek}.

We conclude that by assembling the subnetworks, we can construct the
network $\mathcal{\tilde{D}}$ with $O(K^{3}+K\log^{2}(\epsilon^{-1}))$
weights that provides the required approximation. 
\end{proof}
\begin{proof}[Proof of theorem \ref{thm:3}]
The technique of the proof is similar to the method of proof of Theorem
\ref{thm:2}. We use \eqref{Thm2-3-M} to obtain an estimate for subnetworks
$S_{k,n}$, $1\le k\le K$ each of $O(n)$ weights, satisfying 
\[
\max_{x\in[0,1]}|\sin(2\pi kx)-S_{k,n}(x)|\le ce^{2\pi k-n}.
\]
So, it is possible to construct subnetworks $S_{k}$, $1\le k\le K$,
each with $O(K+\log K+\epsilon^{-1})=O(K+\epsilon^{-1})$ weights
such that 
\[
\max_{x\in[0,1]}|\sin(2\pi kx)-S_{k}(x)|\le\frac{\epsilon}{2K}.
\]
We require a multiplication subnetwork $M$ as in the proof of Theorem
\ref{thm:2}. However, this time, assuming that for sufficiently small
$\epsilon>0$, the outputs of the subnets $\{S_{k}\}_{k=1}^{K}$ are
in $[-2,2]$. Since these are inputs to the multiplication network,
we construct a subnetwork with $O(\log^{2}(K\epsilon^{-1}))$ weights
that satisfies 
\begin{equation}
|M(x_{1},x_{2})-x_{1}x_{2}|\le\frac{\epsilon}{2K},\quad\forall x_{1},x_{2}\in[-2,2],.\label{approx-M-2}
\end{equation}
We assemble the subnetworks to construct an approximating MLP network
with $O(K^{2}+K\log^{2}(K\epsilon^{-1}))$ weights 
\[
\mathcal{\tilde{\mathcal{R}}}(a_{1},\dots,a_{K},x):=\sum_{k=1}^{K}M(a_{k},S_{k}(x)).
\]
Finally, we obtain the required estimate by 
\begin{align*}
 & \left|{\sum_{k=1}^{K}a_{k}\sin(2\pi kx)-\mathcal{\tilde{\mathcal{R}}}(a_{1},\dots,a_{K},x)}\right|=\left|{\sum_{k=1}^{K}a_{k}\sin(2\pi kx)-\sum_{k=1}^{K}M(a_{k},S_{k}(x))}\right|\\
 & \qquad\le\left|{\sum_{k=1}^{K}a_{k}\sin(2\pi kx)-\sum_{k=1}^{K}a_{k}S_{k}(x)}\right|+\left|{\sum_{k=1}^{K}a_{k}S_{k}(x)-\sum_{k=1}^{K}M(a_{k},S_{k}(x))}\right|\\
 & \qquad\le K\frac{\epsilon}{2K}+K\frac{\epsilon}{2K}\le\epsilon.
\end{align*}
\end{proof}
\begin{proof}[Proof of theorem \ref{thm:4}]
It is well known that the Fourier series has the `spectral approximation'
property. Namely, for a Sobolev function $g\in W_{2}^{r}[0,1]$, with
the Fourier expansion $g=\sum_{k=-\infty}^{\infty}\hat{g}(k)e^{2\pi ikx}$,
we can estimate the error of the truncated Fourier expansion: 
\begin{align*}
\left\Vert {g-\sum_{k=-K}^{K}\hat{g}(k)e^{2\pi ikx}}\right\Vert _{2}^{2} & =\left\Vert {\sum_{k=-\infty}^{\infty}\hat{g}(k)e^{2\pi ikx}-\sum_{k=-K}^{K}\hat{g}(k)e^{2\pi ikx}}\right\Vert _{2}^{2}\\
 & =\left\Vert {\sum_{|k|>K}\hat{g}(k)e^{2\pi ikx}}\right\Vert _{2}^{2}\\
 & =\sum_{|k|>K}|\hat{g}(k)|^{2}=\\
 & \leq\sum_{|k|>K}\left({\frac{|2\pi k|}{K}}\right)^{2r}|\hat{g}(k)|^{2}\\
 & \leq K^{-2r}\sum_{k=-\infty}^{\infty}|2\pi k|^{2r}|\hat{g}(k)|^{2}\\
 & =K^{-2r}\sum_{k=-\infty}^{\infty}|\widehat{g^{(r)}}(k)|^{2}=K^{-2r}\|g^{(r)}\|_{2}^{2}.
\end{align*}
Thus, for $g=\sum_{k=1}^{\infty}g_{k}\sin(2\pi kx)$, $g\in W_{2,[0,1]}^{r}$,
$\|g^{(r)}\|_{2}\leq1$, we obtain 
\[
\|g-\sum_{k=-K}^{K}\hat{g}(k)e^{2\pi ikx}\|_{2}\leq K^{-r}.
\]
This implies that for any initial condition function $f=\sum_{k=1}^{\infty}c_{k}\sin(2\pi kx)$,
$f\in W_{2,[0,1]}^{r}$, $\|f^{(r)}\|_{2}\leq1$ and $t\in[0,1]$,
we may approximate the solution $u(f,x,t)$ to the heat equation by
\[
\|u(f,\cdot,t)-u_{K}(\cdot,t)\|_{2}\leq K^{-r},\qquad u_{K}(x,t):=\sum_{k=1}^{K}c_{k}e^{-4\pi^{2}k^{2}t}\sin(2\pi kx).
\]
For the given $\epsilon$ and $r\ge1$, we select 
\[
K:=\left({\frac{3}{\epsilon}}\right)^{1/r},
\]
which gives 
\begin{equation}
\|u(f,\cdot,t)-u_{K}(\cdot,t)\|_{2}\le\frac{\epsilon}{3},\label{u-K-approx}
\end{equation}
uniformly for all initial conditions from our Sobolev ball and all
times $t\in[0,1]$. With this choice of $K$, we construct the following
two networks 
\begin{enumerate}
\item Using Theorem \ref{thm:2}, we may construct for $\tilde{\epsilon}:=\epsilon/(3K)$
a block $\tilde{\mathcal{D}}$ containing 
\[
O(K^{3}+K\log^{2}(\tilde{\epsilon}^{-1}))=O(K^{3}+K\log^{2}(K\epsilon^{-1}))=O(\epsilon^{-3/r}+\epsilon^{-1/r}\log^{2}(\epsilon^{-(1+1/r)}))
\]
weights that satisfies 
\begin{equation}
\|\mathcal{\tilde{\mathcal{D}}}(t,c_{1},...,c_{K})-\mathcal{D}(t,c_{1},...,c_{K})\|_{\infty}\leq\frac{\epsilon}{3K}.\label{D-tilde-est}
\end{equation}
\item Based on \eqref{D-tilde-est}, for sufficiently small $\epsilon>0$,
the output of $\tilde{\mathcal{D}}$ is a vector in $[-2,2]^{K}$,
since it approximates the output of $\mathcal{D}$ which is a vector
in $[-1,1]^{K}$. Using Theorem \ref{thm:3} we may construct a block
$\tilde{\mathcal{R}}$ containing 
\[
O(K^{2}+K\log^{2}(K\epsilon^{-1}))=O(\epsilon^{-2/r}+\epsilon^{-1/r}\log^{2}(\epsilon^{-(1+1/r)}))
\]
weights that satisfies 
\[
|\mathcal{\tilde{\mathcal{R}}}(a_{1},....,a_{K},x)-\mathcal{R}(a_{1},....,a_{K},x)|\leq\frac{\epsilon}{3},\qquad\forall a_{k}\in[-2,2],1\le k\le K,\quad x\in[0,1].
\]
\end{enumerate}
Our approximating solution is then defined by $\tilde{u}(f,x,t):=\mathcal{\tilde{\mathcal{R}}}(\tilde{\mathcal{D}}(t,c_{1},\dots,c_{K}),x)$,
where the network contains a total of $O(\epsilon^{-3/r}+\epsilon^{-1/r}\log^{2}(\epsilon^{-(1+1/r)}))$
weights. We can estimate the approximation by 
\begin{equation}
\|u(f,\cdot,t)-\tilde{u}(f,\cdot,t)\|_{2}\leq\|u(f,\cdot,t)-u_{K}(\cdot,t)\|_{2}+\|u_{K}(\cdot,t)-\tilde{u}(f,\cdot,t)\|_{2}.\label{triangle-U-K}
\end{equation}
Applying \eqref{u-K-approx} provides the bound $\epsilon/3$ for
the first right hand side term in \eqref{triangle-U-K}. We now proceed
to bound the second term by $2\epsilon/3$ using the the simple inequality
$\|g\|_{L_{2}[0,1]}\le\|g\|_{L_{\infty}[0,1]}$. To this end for any
$x\in[0,1]$ 
\begin{align*}
 & |u_{K}(x,t)-\tilde{u}(f,x,t)|=\left|{\mathcal{R}(\mathcal{D}(t,c_{1},...,c_{K}),x)-\tilde{\mathcal{R}}(\tilde{\mathcal{D}}(t,c_{1},...,c_{K}),x)}\right|\\
 & \qquad\le\left|{\mathcal{R}(\mathcal{D}(t,c_{1},...,c_{K}),x)-\mathcal{R}(\tilde{\mathcal{D}}(t,c_{1},...,c_{K}),x)}\right|+\left|{\mathcal{R}(\tilde{\mathcal{D}}(t,c_{1},...,c_{K}),x)-\tilde{\mathcal{R}}(\tilde{\mathcal{D}}(t,c_{1},...,c_{K}),x)}\right|\\
 & \qquad\le K\frac{\epsilon}{3K}+\frac{\epsilon}{3}=\frac{2\epsilon}{3}.\\
\end{align*}
\end{proof}

\end{document}